\documentclass[letterpaper, oneside, 12pt]{article}
\usepackage{amsmath}
\usepackage{graphicx}%
\usepackage{amsfonts}%
\usepackage{amssymb}
\usepackage{overpic}
\usepackage{subfigure}
\usepackage{rotating}
\usepackage{epstopdf}
\usepackage{url}
\usepackage{color}
\usepackage{comment}
\usepackage{enumerate}
\usepackage{mathtools}
\DeclarePairedDelimiter{\ceil}{\lceil}{\rceil}

\usepackage{setspace}
\usepackage[hmargin=1.0in,vmargin=1.0in]{geometry}

\newcommand{\bA}{ {\boldsymbol A} }

\newcommand{\bD}{ {\boldsymbol D} }

\newcommand{\bI}{ {\boldsymbol I} }

\newcommand{\bJ}{ {\boldsymbol J} }

\newcommand{\bO}{ {\boldsymbol O} }

\newcommand{\bx}{ {\boldsymbol x} }
\newcommand{\bX}{ {\boldsymbol X} }
\newcommand{\by}{ {\boldsymbol y} }

\newcommand{\balpha}{ {\boldsymbol \alpha} }

\newcommand{\bbeta}{ {\boldsymbol \beta} }

\newcommand{\bgamma}{ {\boldsymbol \gamma} }
\newcommand{\bGamma}{ {\boldsymbol \Gamma} }

\newcommand{\bepsilon}{ {\boldsymbol \epsilon} }

\newcommand{\bPhi}{ {\boldsymbol \Phi} }

\newcommand{\btau}{ {\boldsymbol \tau}}

\newcommand{\bmu}{ {\boldsymbol \mu} }

\newcommand{\bSigma}{ {\boldsymbol \Sigma} }

\newcommand{\bzero}{ {\boldsymbol 0} }

\newcommand{\given}{\,|\,}

\newtheorem{theorem}{Theorem}[section]

\newtheorem{proposition}[theorem]{Proposition}
\newtheorem{corollary}[theorem]{Corollary}

\newenvironment{proof}[1][Proof]{\begin{trivlist}
\item[\hskip \labelsep {\bfseries #1}]}{\end{trivlist}}

\newcommand{\qed}{\nobreak \ifvmode \relax \else
      \ifdim\lastskip<1.5em \hskip-\lastskip
      \hskip1.5em plus0em minus0.5em \fi \nobreak
      \vrule height0.75em width0.5em depth0.25em\fi}
\title{Bayesian Compressed Regression}
\author{Rajarshi Guhaniyogi and David B. Dunson}
\begin{document}
\maketitle
\begin{abstract}
As an alternative to variable selection or shrinkage in high dimensional regression, we propose to randomly compress the predictors prior to analysis.  This dramatically reduces storage and computational bottlenecks, performing well when the predictors can be projected to a low dimensional linear subspace with minimal loss of information about the response. As opposed to existing Bayesian dimensionality reduction approaches, the exact posterior distribution conditional on the compressed data is available analytically, speeding up computation by many orders of magnitude while also bypassing robustness issues due to convergence and mixing problems with MCMC.  Model averaging is used to reduce sensitivity to the random projection matrix, while  accommodating uncertainty in the subspace dimension.  Strong theoretical support is provided for the approach by showing near parametric convergence rates for the predictive density in the large $p$ small $n$ asymptotic paradigm.  Practical performance relative to competitors is illustrated in simulations and real data applications.
\end{abstract}


\section{Introduction}
With recent technological progress, it is now routine in many disciplines to collect data containing massive numbers of predictors, ranging from thousands to millions or more.  In such settings, it is commonly of interest to consider regression models such as
\begin{equation}\label{eq:motherequation}
\by=\bX\bgamma+\bepsilon,
\end{equation}
where $\bX$ is an $n\times p$ matrix of predictors, $p \gg n$, $\by$ is an $n\times 1$ response vector and $\bepsilon\sim N_n(\bzero,\sigma^2\bI_n)$ is a residual vector.
As traditional techniques such as maximum likelihood cannot be used, a rich variety of alternatives have been proposed
ranging from penalized optimization methods, such as Lasso (Tibshirani, 1996) and elastic net (Zhou et al., 2004), to Bayesian variable selection or shrinkage methods, such as Bayesian Lasso (Park et al., 2008; Hans, 2009), horseshoe (Carvalho et al., 2009, 2010) and generalized double Pareto (Armagan et al., 2012). The Bayesian approach provides a  probabilistic characterization of uncertainty in the high-dimensional regression coefficients and in the resulting predictions, while penalization methods tend to focus on point estimation.  There is a recent literature showing optimality properties of Bayesian variable selection and shrinkage in high-dimensional settings, allowing the number of candidate predictors to increase exponentially with the sample size while imposing sparsity constraints (Jiang, 2007; Castillo and van der Vaart, 2012; Armagan et al., 2012; Strawn et al., 2012; Bhattacharya et al., 2013).

This literature focuses on properties of the exact posterior, and computability problems are encountered in practical applications involving many predictors.  For example, it is well known that computing the posterior under Bayesian variable selection priors is an intractable NP-hard problem, so that one can at best hope for a rough approximation using Markov chain Monte Carlo (MCMC) sampling unless $p$ is small. A commonly used approach is to approximate the posterior with a computationally tractable distribution. This gives rise to variational Bayes approximations, which are popular in other disciplines (Girolami et al., 2006; Titsias et al., 2009) and have recently started to infiltrate the statistical literature (Faes et al., 2011; Ormerod et al., 2012). One notable disadvantage of variational Bayes is that, except in simple cases such as exponential family models, there is a lack of theoretical justification in terms of accuracy of the approximation or other performance metrics.

We propose a new approach for high-dimensional regression problems based on random projections of the scaled predictor vector prior to analysis. This Bayesian compressed regression (BCR) method solves several problems simultaneously.  There is no computational bottleneck, and scaling to enormous $p$ is trivial.  Our approach simply calculates a conjugate Gaussian-inverse gamma posterior for the regression coefficients corresponding to the compressed predictors in parallel for different random projections having different subspace dimensions. Model averaging (Raftery et al. 1997) is then employed to yield a single posterior predictive distribution. Notably, there is no MCMC sampling and one can obtain the predictive distribution extremely rapidly even in problems with huge numbers of predictors.

 An important issue is the question of theoretical justification. Jiang (2007) showed that carefully tailored Bayesian variable selection priors lead to near parametric rates in estimating the predictive distribution $f(y|x)$ in settings involving massive dimensional predictors under near sparsity constraints.  This is an impressive theoretical result, which is consistent with the excellent performance observed for Bayesian variable selection in practice. However, Jiang's result is for the true posterior distribution, which, as mentioned above, cannot be computed accurately in large $p$ problems. We show that our simple Bayesian compressed regression procedure enjoys similar theoretical guarantees, but importantly these guarantees are for a computable method.
In addition, the compressed regression method is expect to have excellent performance not only under variable selection sparsity but also in the sufficient dimensionality reduction setting in which predictors can be projected to a low dimensional linear subspace with minimal loss of information about the response.

 Bayesian compressed regression is inspired by the data squashing literature and recent dramatic success of compressive sensing.  Data squashing compresses a large number of sample points to a smaller representative set, while attempting to yield similar results to an analysis of the full data set. One of the early articles in the data squashing literature is Dumouchel et al. (1999) where the authors suggest constructing a smaller set of pseudo data having matched moments to the original mother data.  Madigan et al. (2002) instead proposed a model-based approach to achieve likelihood-based clustering of a large data set into a smaller number of data points with associated weights. Owen (2003) used empirical likelihood-based data squashing, while Lee et al. (2008) proposes a representative sampling approach.  The idea of compressive sensing instead focuses on randomly compressing the data to facilitate storage and analysis, while retaining the ability to reconstruct the compressed signals with high accuracy under sparsity conditions (Donoho, 2006; Candes et al., 2006).

Our approach is related to compressive sensing in relying on random projections, but there are fundamental differences.  Notably, in compressive sensing, one is not interested in settings involving predictors and a response but instead wants to compress a huge number $n$ of samples of a signal into fewer samples $m$. Our approach is somewhat orthogonal to this in leaving the number of samples $n$ unchanged but instead compressing the number of predictors to a smaller dimension. Zhou, Lafferty and Wasserman (2009) proposed a compressed regression approach to maintain privacy by replacing (\ref{eq:motherequation}) with
\begin{equation}\label{eq:sqlasso}
\bPhi\by=\bPhi\bX\bgamma+\bPhi\bepsilon, \bepsilon\sim N(\bzero,\sigma^2\bI).
\end{equation}
Using an L1 regularization approach to obtain a sparse point estimate, they showed oracle properties. Clearly, this approach also compresses sample size instead of predictors.

Section 2 proposes the Bayesian compressed regression (BCR) approach, while providing background.  Section 3 provides theoretical results on convergence rates.  Section 4 contains a simulation study.  Section 5 contains real data applications, and Section 6 concludes with a discussion. All proofs are included in an Appendix.

\section{Proposed Approach}\label{sec:BRS}

\subsection{Compressing predictors through random projection}

For subjects $i=1,\ldots,n$, let $y_i \in \mathcal{Y}$ denote a response and $\bx_i = (x_{i1},\ldots,x_{ip})' \in \mathcal{X} \subset \Re^p$ denote predictors.  We consider compressed regression models having the general form
\begin{equation}\label{eq:Bayessquash}
y_i \sim f\big\{ (\bPhi\bx_i)'\bbeta,\sigma\big\},
\end{equation}
where $f(\mu,\sigma)$ is a family of distributions having location $\mu$ and scale $\sigma$, $\bPhi$ is an $m \times p$ projection matrix, and $\bbeta = (\beta_1,\ldots,\beta_m)'$ are coefficients on the compressed predictors.  Unlike other dimensionality reduction methods, which rely on projecting high-dimensional predictors to a lower-dimensional subspace, we do not attempt to estimate $\bPhi$ based on the data, as this conveys a daunting computational price.  Instead, we draw the elements $\{ \Phi_{ij} \}$ independently, setting $\Phi_{ij} = -\sqrt{\frac{1}{\psi}}$ with probability $\psi^2$, $0$ with probability $2(1-\psi)\psi$ and $\sqrt{\frac{1}{\psi}}$ with probability $(1-\psi)^2$, with the rows of $\bPhi$ then normalized using Gram-Schmidt orthonormalization. This method of sampling $\Phi$ is popular in compressed sensing (Dasgupta, 2003, 2013). In our approach, we estimate $\psi$, which is restricted to $(0.1,1)$ for numerical stability.

In implementations of Bayesian compressed regression (BCR), we focus on Gaussian linear models, replacing (\ref{eq:motherequation}) with
\begin{equation}\label{eq:Bayessquash1}
y_i=(\bPhi\bx_i)'\bbeta+\epsilon_i,\:\epsilon_i\sim N(0,\sigma^2),
\end{equation}
where regression coefficients $\bbeta = (\beta_1,\ldots,\beta_m)'$ have a much lower dimension and different interpretation than the regression coefficients in (\ref{eq:motherequation}). Since (\ref{eq:Bayessquash1}) is a normal linear regression model with $m<n$, we are no longer in the high-dimensional setting and can choose usual conjugate priors for $(\bbeta,\sigma)$.  In particular, we choose a normal-inverse gamma (NIG) prior,
\begin{equation*}
(\bbeta\given\sigma^2) \sim N(\bzero,\sigma^2\bSigma_{\bbeta}),\:\sigma^2\sim IG(a,b),
\end{equation*}
leading to an NIG posterior distribution for $(\bbeta,\sigma)$ given $\by$ and $\bPhi\bx$.  In the special case in which $a,b \to 0$, we obtain Jeffrey's prior and the posterior distribution is
\begin{align}\label{eq:postdist}
\bbeta\given\by &\sim t_{n}(\bmu,\bSigma),\\
\sigma^2\given\by &\sim IG(a_1,b_1),
\end{align}
where $a_1=n/2$, $b_1=\left\{\by'\by-\by'\bX\bPhi'\left[\bPhi\bX'\bX\bPhi'+\bSigma_{\bbeta}^{-1}\right]^{-1}\bPhi\bX'\by\right\}/2$,\\
$\bmu=\left[\bPhi\bX'\bX\bPhi'+\bSigma_{\bbeta}^{-1}\right]^{-1}\bPhi\bX'\by$, $\bSigma=(2b_1/n)\left[\bPhi\bX'\bX\bPhi'+\bSigma_{\bbeta}^{-1}\right]^{-1}$,
and $t_{\nu}(\bmu,\bSigma)$ denotes a multivariate-$t$ distribution with $\nu$ degrees of freedom, mean $\bmu$ and covariance $\bSigma$.

Hence, the exact posterior distribution of $(\bbeta,\sigma^2)$ conditionally on $\bPhi$ is available analytically.  In practice we are not interested in inferences on  $\bbeta$ directly, but instead would like to do predictions or inferences on $\bgamma$.  The predictive of $y_{n+1}$ given $\bx_{n+1}$ and $\bPhi$ for a new $(n+1)$st subject marginalizing out $(\bbeta,\sigma^2)$ over their posterior distribution is
\begin{equation}\label{eq:fyxPhi}
y_{n+1} | \bx_{n+1}, \by\sim t_{n}\left(\mu_{pred},\sigma_{pred}^2\right),
\end{equation}
where $\mu_{pred}=(\bPhi\bx_{n+1})'\bmu$, $\sigma_{pred}^2=(2b_1/n)\left[1+(\bPhi\bx_{n+1})'\left(\bSigma_{\bbeta}^{-1}+
\bPhi\bX'\bX\bPhi'\right)^{-1}\bPhi\bx_{n+1}\right]$.

To provide a heuristic motivation for BCR, we rewrite (\ref{eq:Bayessquash1}) as
\begin{equation*}
y_i=\bx_i'\bPhi'\bbeta+\epsilon_i = \bx_i'\bgamma+\epsilon_i,\:\epsilon_i \sim N(0,\sigma^2).
\end{equation*}
Assigning a joint NIG prior on $(\bbeta,\sigma^2)$ induces a prior on $(\bgamma,\sigma^2)$ conditionally on the randomly generated $\bPhi$, which is kept fixed at its initial generated value and not updated during data analysis.  The prior on $\bgamma$ is clearly a singular distribution that resides on an $m$-dimensional hyperplane embedded in $\mathcal{R}^p$. Typical Bayesian approaches for high-dimensional linear regression would instead directly define a prior for $\bgamma$ that concentrates on a low-dimensional linear subspace; BCR accomplishes this indirectly through random projections.  By avoiding updating of $\bPhi$ and only learning the posterior for $\bbeta$, we obtain enormous computational savings while maintaining superior predictive performance in our experience and justified theoretically in Section 3.

On the surface, (\ref{eq:Bayessquash1}) seems reminiscent of sufficient dimensionality reduction (SDR, Cook, 1998), which attempts to find the smallest subspace $\mathcal{S}$
having basis  $\bA\in \mathcal{R}^{p\times d}$, $d\ll p$, satisfying $f(\by\given\bX)=f(\by\given\bA'\bX)$. SDR-based approaches simultaneously estimate $\bA$ and the density
$f(\by\given\bA'\bX)$ relying on a rich variety of strategies, all of which face severe computational bottlenecks in estimating $\bA$ for even moderately large data sets. For example, the elegant Bayesian approach of Ghosh et al. (2010) cannot cope with more than a few dozens of predictors and a few dimensional subspace. Instead we free this bottleneck by randomly generating $\bPhi$ in advance of the data analysis, and then rely on model averaging to reduce sensitivity.

\subsection{Model averaging}

The approach described in the previous section can be used to obtain a posterior distribution for $\bgamma$ and a predictive distribution for $y_{n+1}$ given $\bx_{n+1}$ for a new $(n+1)$st subject {\em conditionally} on the $m \times p$ random projection matrix $\bPhi$.  We would like to limit sensitivity of the results to the specified $m$ and randomly generated $\bPhi$. This is accomplished by generating $s$ random projection matrices having different $(m,\psi)$ values, and then using model averaging to combine the results.  We let $\mathcal{M}_l$, $l=1,\ldots,s$, represent (\ref{eq:Bayessquash1}) with $\bPhi$ having $m_l$ rows and $\psi_l\sim U(0.1,1)$. Corresponding to model $\mathcal{M}_l$, we denote $\bPhi$, $\bbeta$ and $\sigma^2$  by $\bPhi^{(l)}$, $\bbeta^{(l)}$ and $\sigma^{2(l)}$ respectively.  Let $\mathcal{M} = \{ \mathcal{M}_1,\ldots, \mathcal{M}_s \}$ denote the set of models corresponding to different random projections, $\mathcal{D} = \{ (y_i, \bx_i), i=1,\ldots,n \}$ denote the observed data, and $(y_{n+1},\bx_{n+1})$ denote the data for a future subject.  The predictive density of $y_{n+1}$ given $\bx_{n+1}$ is
\begin{equation}\label{eq:modelaver}
f(y_{n+1} | \bx_{n+1}, \mathcal{D}) = \sum_{l=1}^s f(y_{n+1} | \bx_{n+1}, \mathcal{M}_l, \mathcal{D})P(\mathcal{M}_l\given \mathcal{D}),
\end{equation}
where the predictive density of $y_{n+1}$ given $\bx_{n+1}$ under projection $\mathcal{M}_l$ is given in (\ref{eq:fyxPhi})
and the posterior probability weight on projection $\mathcal{M}_l$ is
\begin{equation*}
P(\mathcal{M}_l\given \mathcal{D})=\frac{P(\mathcal{D}\given \mathcal{M}_l)P(\mathcal{M}_l)}{\sum_{h=1}^s P(\mathcal{D}\given \mathcal{M}_h)P(\mathcal{M}_h)}.
\end{equation*}
Assuming equal prior weights $P(\mathcal{M}_l)=1/s$.  The marginal likelihood under $\mathcal{M}_l$ is
\begin{equation}\label{eq:marg_model}
P(\mathcal{D}\given\mathcal{M}_l)=\int P(\mathcal{D}\given\mathcal{M}_l,\bbeta^{(l)},\sigma^{2(l)})\pi(\bbeta^{(l)},\sigma^{2(l)})
d\bbeta^{(l)}d\sigma^{2(l)}.
\end{equation}
 After a little algebra, one observes that for
(\ref{eq:Bayessquash1}) with $(\bbeta\given\sigma^2)\sim N(\bzero,\sigma^2\bSigma_{\bbeta})$, $\pi(\sigma^2)\propto \frac{1}{\sigma^2}$,
\begin{equation*}
P(\mathcal{D}\given\mathcal{M}_l)=\frac{1}{\left|\bX\bPhi'\bSigma_{\bbeta}^{(l)}\bPhi\bX'+\bI\right|^{\frac{1}{2}}}
\frac{2^{\frac{n}{2}}\Gamma(\frac{n}{2})}{\left[\by'\left(\bX\bPhi'\bSigma_{\bbeta}^{(l)}\bPhi\bX'+\bI\right)^{-1}\by\right]^{\frac{n}{2}}(\sqrt{2\pi})^{n}}.
\end{equation*}
Plugging in the above expressions in (\ref{eq:modelaver}), one obtains the posterior predictive distribution as a weighted average of $t$ densities.
 Invoking the Sherman-Woodbury-Morrison matrix identity and matrix determinant lemma, one obtains
\begin{align*}
\left(\bX\bPhi'\bSigma_{\bbeta}\bPhi\bX'+\bI\right)^{-1}&= \bI-\bX\bPhi'\left(\bSigma_{\bbeta}^{-1}+\bPhi\bX'\bX\bPhi'\right)^{-1}\bPhi\bX'\\
|\bX\bPhi'\bSigma_{\bbeta}\bPhi\bX'+\bI| &= |\bSigma_{\bbeta}^{-1}+\bPhi\bX'\bX\bPhi'||\bSigma_{\bbeta}|.
\end{align*}
  With the help of the above identity, different components in (\ref{eq:marg_model}) can be estimated in parallel with the main computational expense being $O(m_l^3)$ matrix inversions under the $l$th random projection.  As $m_l \ll p$ such inversion can be obtained quickly.  On a cluster, one can easily average across a massive number $s$ of possible $\bPhi$s.  However, we have found that there are diminishing gains after a modest number (e.g., $s \sim 100$) and hence BCR can be implemented very rapidly even using a batch implementation in R or Matlab.

An important question that remains is how much information is lost in compressing the high-dimensional predictor vector to a much lower dimension? In particular, one would expect to pay a price for the huge computational gains in terms of predictive performance or other metrics.  We address this question in two ways.  First we consider the theoretical performance in prediction in a large $p$, small $n$ asymptotic paradigm in Section 3.  Then, we will consider practical performance in finite samples using simulated and real data sets.

\section{Convergence Rates of Predictive Densities}\label{sec:convrate}

In this section we study the convergence properties of BCR.  Our development follows that of Jiang (2007), with some important differences.  He studied the convergence rate of the predictive distribution in high-dimensional regression models under near sparsity conditions using Bayesian variable selection.  His results on near parametric convergence rates are for the true posterior distribution, which is not computable.  Instead we focus on obtaining corresponding results for BCR without model averaging, and hence study the large $p$, small $n$ asymptotic performance of a posterior that is computable exactly even for massive $p$.

Let $f_0$ denote the true conditional density of the response $y$ given predictors $\bx$ and let $f$ be the random predictive density that we obtain a posterior for.  Define the Hellinger distance between $f$ and $f_0$ as $d(f,f_0)=\sqrt{\int\int (\sqrt{f}-\sqrt{f_0})^2 \nu_y(dy)\nu_x(dx)}$, where $\nu_x(dx)$ is the unknown probability measure for $x$ and $\nu_y(dy)$ is the dominating measure for $f$ and $f_0$.  Our convergence rate results focus on
\begin{equation}\label{eq:convergence}
E_{f_0}\pi\left[d(f_0,f)>\epsilon_n\given (y_i,\bx_i)_{i=1}^{n}\right]<\lambda_n,
\end{equation}
for large enough $n$, and some sequences $\epsilon_n,\lambda_n$ converging to 0 as $n\rightarrow\infty$.  This shows that the posterior probability assigned outside of a shrinking neighborhood around the true predictive density $f_0$ decreases to zero.  Ideally, this result would hold even when the number of predictors increases more rapidly than the sample size and when $\epsilon_n$ decreases rapidly.  In particular, we seek to establish a convergence rate $\epsilon_n$ of the order of $n^{-1/2}$ up to a $\log(n)$ factor for the proposed model. Below we describe basic notations to be used throughout this section.

\subsection{Notation and Framework}

Letting $p_n$ denote the number of predictors for sample size $n$, we assume that $p_n$ is a non-decreasing sequence of $n$.  We also let the subspace dimension $m_n$ grow with sample size, allowing that there may be additional signal discovered as more predictors are considered.  We assume all the predictors are standardized, $|x_j|<1$ for all $j$. The true density and the predictive density of the fitted model are assumed to lie in the class of generalized linear models with only
parameter $\bbeta$.  We write
\begin{equation}\label{eq:truth}
f_0(u_0)=\exp\left\{y a(u_0)+b(u_0)+c(y)\right\};\:\: u_0=\bx'\bbeta_0
\end{equation}
as the true density and
\begin{equation}\label{eq:preddens}
y\sim f(u)=\exp\left\{y a(u)+b(u)+c(y)\right\};\:\: u=(\bPhi\bx)'\bbeta
\end{equation}
as the density of the fitted model, where $a(z)$ and $b(z)$ are continuously differentiable functions with $a(z)$ having nonzero derivatives and $\bPhi$ is an $m_n\times p_n$ matrix. This parametrization includes some popular classes of densities including
binary probit and logistic regression, linear regression with constant variance and so on.
The $\bPhi$ is standardized, so that
$||\bPhi\bx||<||\bx||$, $\forall\:\bx$. Although $\bPhi$ grows in size with $n$, we suppress the dependence on $n$
 for notational clarity.

Corresponding to the true predictive $f_0$, there is a vector of true regression parameters $\bbeta_0$. We assume a near sparsity condition on
$\bbeta_0$:  $\lim_{n\rightarrow\infty}\sum_{j=1}^{p_n}|\beta_{j0}|<\infty$, implying that many elements are small in magnitude.  This is a more appealing and weaker condition than the more standard exact sparsity assumption.

Suppose with $n$ samples we observe covariates $\bx_1,\ldots,\bx_n$. We will use the empirical measure that puts $\frac{1}{n}$ probability on each $\bx_1$,...,$\bx_n$ as the dominating measure $\nu_{\bx}$ on $\bx$. The dominating measure of $(\bx,y)$ is taken to be the product of $\nu_{\bx}(\bx)\nu_y(y)$ so that the true joint density of $(\bx,y)$ is $f_0(\bx'\bbeta_0)\nu_{\bx}(\bx)\nu_y(y)$.

In studying theoretical properties, we focus on the broader class of GLMs instead of just normal linear regression and additionally assume there is no free scale parameter, a standard assumption in the literature on high-dimensional regression. Following standard convention, we let $\sigma^2 = 1$ without loss of generality.  In addition, we consider two alternative priors for $\bbeta$, with the first letting $\bbeta \sim N(\bzero,\bSigma_{\bbeta})$ and the second corresponding to independent $\beta_j \sim DE(1)$.  In both cases, under some assumptions, it can be shown that the posterior predictive densities achieve near parametric convergence rates.

\subsection{Main Results}

This section describes our main results.

\begin{theorem}\label{theorem1}
 Let $\bbeta \sim \mbox{N}(\bzero,\bSigma_{\bbeta})$ apriori and $\tilde{B}_n$ and $\underline{B}_n$ be the largest and the smallest eigenvalues
 of $\Sigma_{\bbeta}$. Further assume all the covariates are standardized, i.e. $|x_j|<1$ and  $\lim_{n\rightarrow\infty}\sum_n |\beta_{j0}|<K$. Define $D(R)=1+R\sup\limits_{|h|\leq R}|a'(h)|\sup\limits_{|h|\leq R}\left|\frac{b'(h)}{a'(h)}\right|$, $\theta_n=\sqrt{m_n p_n}$. For a sequence $\epsilon_n$ satisfying $0<\epsilon_n^2<1,\:n\epsilon_n^2\rightarrow\infty$, assume the following to hold
 \begin{align}
(i)\: \frac{m_n\log(1/\epsilon_n^2)}{n\epsilon_n^2} &\rightarrow 0,\:\:\frac{log(m_n)}{n\epsilon_n^2}\rightarrow 0,\:
 \frac{m_n\log\: D(\theta_n\:\sqrt{8\tilde{B}_n n\epsilon_n^2})}{n\epsilon_n^2} \rightarrow 0\\
(ii)\: \tilde{B}_n\leq Bm_n^v,\:\: & \underline{B}_n\geq B_1(\log(m_n))^{-1}\\
 (iii)\: \frac{\log(||\bPhi\bx||)}{n\epsilon_n^2}\rightarrow 0,\:\: &  ||\bPhi\bx||^2>8\frac{(K^2+1)}{B_1}\frac{\log(m_n)}{n\epsilon_n^2},\:\:\forall\:\bx=\bx_1,...,\bx_n
 \end{align}
 then
 \begin{equation}\label{eq:res1}
 E_{f_0}\pi\left[d(f,f_0)>4\epsilon_n\given(y_i,\bx_i)_{i=1}^{n}\right]\leq 4e^{-n\epsilon_n^2/2}\:\:\mbox{for all large}\:n.
 \end{equation}
\end{theorem}
\qed
The conditions in (i) are primarily designed to impose a restriction on the size of the model, so that the subspace dimension $m_n$ cannot grow too rapidly with $n$.
The constraint on the growth of $D(\theta_n\:\sqrt{8\tilde{B}_n n\epsilon_n^2})$ is, however, difficult to interpret immediately. A close inspection tells us that the rate at which $D(R)$ grows is solely dependent on the rate of growth of $a'(z)$ and $\frac{b'(z)}{a'(z)}$. For logit, probit and  linear regressions with known error variance, $|a'(z)|$ and $|\frac{b'(z)}{a'(z)}|$ at most grow linearly with $|z|$ (Jiang, 2007). Therefore, some additional restrictions are imposed
on the growth of the number of predictors and subspace dimension.

The conditions in (ii) impose some constraints on the prior covariance matrix of $\bbeta$. It is evident that these conditions are quite relaxed and can even be satisfied with naive choices such as $\bSigma_{\bbeta}=\bI$.

The conditions in (iii) characterize the class of feasible matrices $\bPhi$, restricting upper and lower bounds of $||\bPhi\bx||$ for all observed covariates. Intuitively, compression with $\bPhi$ should not take away the power of the covariates to explain response. It is not clear how to choose a matrix $\bPhi$ deterministically that satisfies condition (iii). This suggests generating a random matrix $\bPhi$ that satisfies condition (iii) with high probability, which is the same approach taken in the compressive sensing theory literature.  To avoid needless complexities in the proof, we assume that condition (iii) holds with probability one.

Below, we state the second result on the convergence rate with $DE(1)$ priors on the components of $\bbeta$.

\begin{theorem}\label{theorem2}
 Let $\beta_j$'s be assigned $DE(1)$ apriori. Define $D(R)=1+R\sup\limits_{|h|\leq R}|a'(h)|\sup\limits_{|h|\leq R}\left|\frac{b'(h)}{a'(h)}\right|$, $\theta_n=\sqrt{m_n p_n}$. Further assume that for a sequence $0<\epsilon_n<1$ satisfying $n\epsilon_n^2\rightarrow\infty$, one has
 \begin{align*}
&(i) \frac{m_n\log(1/\epsilon_n^2)}{n\epsilon_n^2} \rightarrow 0,\: \frac{log(m_n)}{n\epsilon_n^2} \rightarrow 0,\:\frac{m_n\log D(4\theta_n\:n\epsilon_n^2)}{n\epsilon_n^2} \rightarrow 0\\
&(ii) \frac{\log(||\bPhi\bx||)}{n\epsilon_n^2}\rightarrow 0,\: ||\bPhi\bx||^2>8\frac{K^2+1}{n\epsilon_n^2}\:\:\forall\:\bx=\bx_1,...,\bx_n\\
& (iii)\lim_{n\rightarrow\infty}\sum_n |\beta_{j0}| <K,
 \end{align*}
 then
 \begin{equation}\label{eq:res1}
 E_{f_0}\pi\left[d(f,f_0)>4\epsilon_n\given(y_i,\bx_i)_{i=1}^{n}\right]\leq 4e^{-n\epsilon_n^2/2}\:\:\mbox{for all large}\:n.
 \end{equation}
\end{theorem}
\qed
These conditions are similar to those of Theorem \ref{theorem1} and hence we omit further discussion.

From Theorem \ref{theorem1}, \ref{theorem2} and the discussions that follow, it is evident that the convergence rate will be highly dependent on the
rate at which $p_n$ and $m_n$ grow with $n$. Intuitively, a good convergence rate should require some control on the number of non-informative
predictors. This in turn implies that $p_n$ should be bounded by some function of $n$. As far as $m_n$ is concerned, the theory shows that $m_n$ cannot grow above a certain limit. The lower bound on the size of $m_n$ is controlled by the complex dependence of $m_n$ on $\bPhi$ and the predictors through condition (iii). All of these considerations are put together in Corollary \ref{corollary1} to obtain a near parametric convergence rate for the proposed BCR approach. The proof of this corollary
relies on routine algebraic manipulations and is thus omitted.

\begin{corollary}\label{corollary1}
 Consider linear regression, logistic regression or probit regression examples. Assume that $\bbeta$ is assigned a $N(\bzero,\bSigma_{\bbeta})$ prior with
 the largest and smallest eigenvalues of $\bSigma_{\bbeta}$, $\tilde{B}_n,\underline{B}_n$ respectively, satisfying
 $\tilde{B}_n\leq Bm_n^v$, $\underline{B}_n\geq B_1\log(m_n)^{-1}$, for all large enough $n$, for some positive constants $B$, $B_1$ and $v$. Suppose
 $|x_j|\leq 1$ for all $j$. Assume further that $p_n\leq \exp(Cn^{\zeta})$ for some $C>0$ and some $\zeta\in (0,1)$, for all large enough $n$
 and $\lim\limits_{n\rightarrow\infty}\sum\limits_{j=1}^{p_n}|\beta_{j0}|<K<\infty$.
 Assume the conditions on the matrix $\bPhi$, as outlined in Theorem \ref{theorem1} are satisfied for large enough n and the number of rows $m_n$ of $\bPhi$ satisfies
 $\frac{m_n}{\log(n)^{k_1}}\rightarrow 0$, for some $k_1>0$ for all large n. Then we can take the convergence rate in Theorem \ref{theorem1} as
 \begin{equation*}
 \epsilon_n\sim n^{-(1-\zeta)/2}\log(n)^{(k_1+1)/2}.
 \end{equation*}
 \end{corollary}

A similar result follows from Theorem \ref{theorem2}.
 \begin{corollary}\label{corollary2}
 Consider linear regression, logistic regression or probit regression examples. Assume that $\beta_j$'s are assigned independently $DE(1)$ prior. Suppose
 $|x_j|\leq 1$ for all $j$. Assume further that $p_n\leq \exp(Cn^{\zeta})$ for some $C>0$ and some $\zeta\in (0,1)$, for all large enough $n$
 and $\lim\limits_{n\rightarrow\infty}\sum\limits_{j=1}^{p_n}|\beta_{j0}|<K<\infty$.
 Assume the conditions on the matrix $\bPhi$, as outlined in Theorem \ref{theorem2} are satisfied for large enough n and the number of rows $m_n$ of $\bPhi$ satisfies
 $\frac{m_n}{(log(n))^{k_1}}\rightarrow 0$, for some $k_1>0$ for all large n. Then we can take the convergence rate in Theorem \ref{theorem2} as
 \begin{equation*}
 \epsilon_n\sim n^{-(1-\zeta)/2}\log(n)^{(k_1+1)/2}.
 \end{equation*}
 \end{corollary}
 \qed

\section{Simulation Study}

In this section we compare the out-of-sample predictive performance of model averaged BCR to that of Ridge Regression (RR), Lasso (Tibshirani, 1996),
partial least squares regression (PLSR), Bayesian Lasso (BL; Park et al, 2008), generalized double Pareto (GDP; Armagan et al, 2012), and Bridge regression (BR).  We also consider an alternative implementation of our compression idea in which instead of using conjugate NIG priors with model averaging over the projection matrix, we generate a single projection matrix, with shrinkage priors specified for the coefficients on the compressed predictors.  Following this strategy, we applied compressed versions of Bayesian Lasso (CBL) and generalized double Pareto (CGDP).   These methods are slower than BCR in relying on MCMC but are massively faster than applying MCMC with shrinkage priors to the original data when $p$ is enormous.  As a default in these analyses, we use $m=40$, which seems to be a reasonable choice of upper bound for the dimension of the linear subspace to compress to.

To implement BCR, we set $\bSigma_{\bbeta}$ to be the identity matrix, which satisfies the restrictions in Corollary \ref{corollary1}. The model averaging step in BCR requires choice of a window over the possible dimensions $m$.
Motivated by the theory in Section~\ref{sec:convrate}, we choose the window as $[\ceil{2*log(p)},min(n,p)]$ which implies that the number of possible models to be averaged across is $s=min(n,p)-\ceil{2*log(p)}$. For MCMC based model implementations, we discard the first 2,500 samples as a burn-in and draw inference based on 7,500 samples.

{\noindent \emph{Moderate dimension cases}}\\
We first consider moderately large $p$ and $n$ cases, assessing how sparsity and changing number of samples impact performance.  We generate  observations
from the standard linear regression model with $p=100$ predictors generated from the normal distribution having $\mbox{cor}(x_j, x_j')=0.5^{|j-j'|}$.
We consider the following scenarios.
\newline
\newline
\emph{Model 1}: First 5 regression coefficients are 1.2, the rest are zero and $\sigma^2=1,\:n=70$.\\
\emph{Model 2}: First 5 regression coefficients are 1.2, the rest are zero and $\sigma^2=1,\:n=110$.\\
\emph{Model 3}: First 15 regression coefficients are 1, the rest are zero and $\sigma^2=1,\:n=70$.\\
\emph{Model 4}: First 15 regression coefficients are 1, the rest are zero and $\sigma^2=1,\:n=110$.\\
\emph{Model 5}: All the regression coefficients are $0.2$ and $\sigma^2=1,\:n=70$.\\
\emph{Model 6}: All the regression coefficients are $0.2$ and $\sigma^2=1,\:n=110$.
\newline
\newline
The last two scenarios are referred to as the \emph{dense} case. Although we have been focusing on settings in which sparsity is justified in our theory and motivation, it is also instructive to compare performance of BCR to its competitors in the more general case in which the $p$-dimensional predictors can be compressed to a much lower dimensional linear subspace without loosing much information about the response $y_i$.  These last two cases correspond to a one dimensional subspace with no sparsity. \emph{Dense} cases are motivated by the practical applications where each of the covariates has small effect on the outcome.

In our experiments $\by$ and $\bX$ are centered and the columns of $\bX$ are standardized to have unit variance. To implement LASSO, RR, BR and PLSR we used \texttt{lars} (Hastie et al., 2012), \texttt{MASS} (Ripley et al., 2012), \texttt{monomvn} (Gramacy, 2010) and \texttt{pls} packages in \texttt{R} respectively. As a default choice suggested in Armagan et al (2012), we fix hyperparameters $\alpha=\eta=1$ for GDP and CGDP. In BL and CBL we put a Gamma($1,1$) prior on the Lasso penalty.

In each of the six scenarios, we simulate 100 datasets. Table~\ref{tab1} presents the MSPE averaged over these simulated datasets where in each dataset MSPE is calculated over the same number of held-out observations as the number of training cases. The values in the subscripts represents bootstrap standard errors for the averaged MSPEs. This is calculated by generating 500 bootstrap samples from the 100 MSPE values, finding averaged MSPE in each of these 100 datasets, and then computing its standard error.
\begin{table}[!th]
{\scriptsize
\begin{center}
\caption{Out of sample $MSPE\times .1$ for the competing approaches with bootstrap $se\times .1$ in the subscript, with columns 1-6 corresponding to results under Models 1-6, respectively.}\label{tab1}
\begin{tabular}
[c]{ccccccc}%
\hline
&\multicolumn{2}{c}{Sparsity level 5} &\multicolumn{2}{c}{Sparsity level 15} & \multicolumn{2}{c}{Dense model}\\
\cline{2-7}
 & $n=70$ & $n=110$ & $n=70$ & $n=110$ & $n=70$ & $n=110$\\
\cline{2-7}
          & & & & & &\\
BCR       & $0.93_{0.023}$ & $0.47_{0.019}$ & $1.54_{0.074}$ & $0.46_{0.01}$ & $0.12_{0.002}$  & $0.11_{0.004}$\\
          & & & & & &\\
CGDP       & $1.23_{0.034}$ & $0.91_{0.042}$ & $2.37_{0.12}$ & $1.76_{0.08}$ & $0.24_{0.012}$ & $0.21_{0.021}$\\
          & & & & & &\\
CBL      & $1.08_{0.030}$ & $0.85_{0.040}$ & $2.09_{0.10}$ & $1.68_{0.07}$ & $0.22_{0.011}$ & $0.20_{0.018}$\\
          & & & & & &\\
GDP        & $0.39_{0.011}$ & $0.36_{0.013}$ & $0.46_{0.029}$ & $.36_{0.01}$ & $0.52_{0.010}$ & $0.47_{0.024}$\\
          & & & & & &\\
BL       & $0.30_{0.007}$ & $0.21_{0.006}$ & $0.40_{0.018}$ & $0.23_{0.007}$ & $0.39_{0.007}$ & $0.29_{0.011}$\\
          & & & & & &\\
BR       & $0.69_{0.010}$ & $0.25_{0.008}$ & $1.36_{0.060}$ & $0.28_{0.009}$ & $0.45_{0.008}$ & $0.21_{0.007}$\\
         &  &  &  & & & \\
LASSO     & $0.13_{0.003}$ & $0.25_{0.029}$ & $0.19_{0.011}$ & $0.70_{0.068}$ & $0.49_{0.008}$ & $0.64_{0.040}$\\
          & & & & & &\\
RR        & $0.42_{0.009}$ & $0.25_{0.011}$ & $0.51_{0.020}$ & $0.25_{0.008}$ & $0.37_{0.007}$ & $0.23_{0.025}$\\
          & & & & & &\\
PLSR      & $0.34_{0.007}$& $0.22_{0.008}$& $0.51_{0.023}$& $0.29_{0.010}$ & $0.22_{0.004}$ &$0.17_{0.005}$\\
\hline
\end{tabular}
\end{center}
}
\end{table}

From Table~\ref{tab1} it is evident that in the dense cases (\emph{Model 5},\emph{Model 6}), performance of BCR is significantly better than the competing shrinkage and sparsity inducing approaches. Additionally, BCR yields remarkably better MSPE than  partial least square regression (PLSR) under \emph{Model 5}, which particularly favors the usage of PLSR. In sparse cases (\emph{Model 1},\emph{Model 2}) all the sparsity favoring approaches such as LASSO, BL, GDP work better than BCR.

As sample size increases, performance of BCR improves along with its competitors. For cases with higher sample size (\emph{Model 2},\emph{Model 4}), BCR shows competitive performance with GDP and BR. We repeated all the experiment after increasing signal to noise ratio and found similar ordering in their performances.

\begin{table}[!th]
{\scriptsize
\begin{center}
\caption{Median lengths of 95\% predictive intervals for the competing approaches.}\label{tab3}
\begin{tabular}
[c]{ccccccc}%
\hline
&\multicolumn{2}{c}{Sparsity level 5} &\multicolumn{2}{c}{Sparsity level 15} & \multicolumn{2}{c}{Dense model}\\
\cline{2-7}
 & $n=70$ & $n=100$ & $n=70$ & $n=100$ & $n=70$ & $n=100$\\
\cline{2-7}
          & & & & & &\\
BCR       & $6.06(4.99,7.22)$ & $4.49(3.54,6.46)$ & $7.89(6.45,9.61)$ & $5.30(4.65,5.96)$ &  $4.06(3.43,4.83)$& $4.02(3.56,4.68)$\\
          & & & & & &\\
CGDP       & $10.46(8.52,12.61)$ & $10.57(7.27,12.01)$ & $13.90(10.91,18.29)$ & $14.43(10.82,17.20)$ & $4.70(3.58,7.85)$ & $4.35(3.71,7.57)$\\
          & & & & & & \\
CBL      & $10.98(9.18,12.92)$ & $10.90(7.63,12.32)$ & $14.63(11.83,18.90)$ & $14.98(11.38,17.83)$ & $5.01(3.82,8.20)$ & $4.52(3.86,7.89)$\\
          & & & & & & \\
GDP        & $3.92(2.87,5.51)$ & $4.00(2.87,4.95)$ & $5.33(3.11,6.97)$ & $4.66(3.49,5.50)$ & $6.16(4.81,7.59)$ & $4.56(3.39,5.27)$\\
          & & & & & & \\
BL       & $4.89(3.77,6.12)$ & $4.56(3.68,5.20)$ & $7.04(5.75,7.95)$  & $5.58(4.44,6.36)$ & $6.06(5.09,6.91)$ & $5.31(4.38,5.88)$\\
          & & & & & & \\
BR       & $17.16(15.49,18.72)$ & $6.25(4.79,7.10)$ & $21.16(19.96,23.17)$ & $7.27(5.88,8.25)$ & $14.57(12.37,16.38)$ & $5.68(4.75,6.55)$\\
         &  & & & & &\\
LASSO   & $2.70(1.32,3.46)$  & $5.85(3.59,9.72)$ & $2.17(0.86,2.72)$ & $8.86(4.01,13.51)$ & $0.77(0.01,2.84)$ & $6.28(1.79,9.82)$\\
        &   &  & & & &\\
RR      & $0.03(0.02,0.05)$  & $2.18(0.80,2.82)$ & $0.04(0.02,0.06)$ & $2.25(1.35,3.01)$ & $0.05(0.02,0.06)$ & $2.29(0.77,2.62)$\\
        &   &  &  & & &\\
PLSR   &  $2.76(2.31,3.29)$  & $3.05(2.71,3.51)$ &  $3.26(2.80,3.86)$ & $3.64(3.09,4.05)$ & $2.37(1.87,2.83)$ &$2.82(2.51,3.30)$\\
\hline
\end{tabular}
\end{center}
}
\end{table}

Figure~\ref{syn-surfs} shows empirical coverage probabilities of 95\% predictive intervals for the competing Bayesian models in six scenarios.  BCR and GDP have under-coverage in the sparse cases with low sample size (\emph{Model 1},\emph{Model 3}) while the other compressed models have satisfactory coverage. BCR, however, shows excellent coverage in the dense case with low sample size.  For the frequentist point estimation approaches, we use a two stage approach:  (i) estimate regression coefficients in the first stage; (ii) construct 95\% PIs  based on the normal distribution centered on the predictive mean from the regression model with variance equal to the estimated variance in the residuals. The median coverage of LASSO and PLSR  are found to be 74\% and 54\% respectively in \emph{Model 1}  while RR shows severe under-coverage. Coverage of LASSO becomes much worse (55\% for \emph{Model 3} and 13\% for \emph{Model 5}) as sparsity decreases. The coverage of PLSR increases marginally as the sparsity decreases.  As sample size increases, predictive coverage of BCR improves. In the dense case with higher sample size, BCR produces coverage closest to 95\%. The coverage of LASSO, PLSR and RR also increases with increasing sample size.

The median length of 95\% PIs for each of the methods
are provided in Table~\ref{tab3}. Among the competing Bayesian methods, BCR and GDP have the shortest 95\% predictive intervals in sparse cases while all the compressed models have wider predictive intervals. This explains under-coverage of BCR and GDP compared to the other competitors. In the dense case, compressed approaches have narrow predictive intervals and better coverage.  The frequentist intervals are narrower in general, likely as a result of ignoring uncertainty in parameter estimation. With increasing sample size, PI's for BCR shrinks in size while maintaining better coverage. It is promising that BCR has such competitive performance even without considering computation time; in the next simulation cases, we consider much higher dimensional cases and computation time comparisons.

\begin{figure}[!ht]
  \begin{center}
    \subfigure[\emph{Model 1}]{\includegraphics[width=8cm]{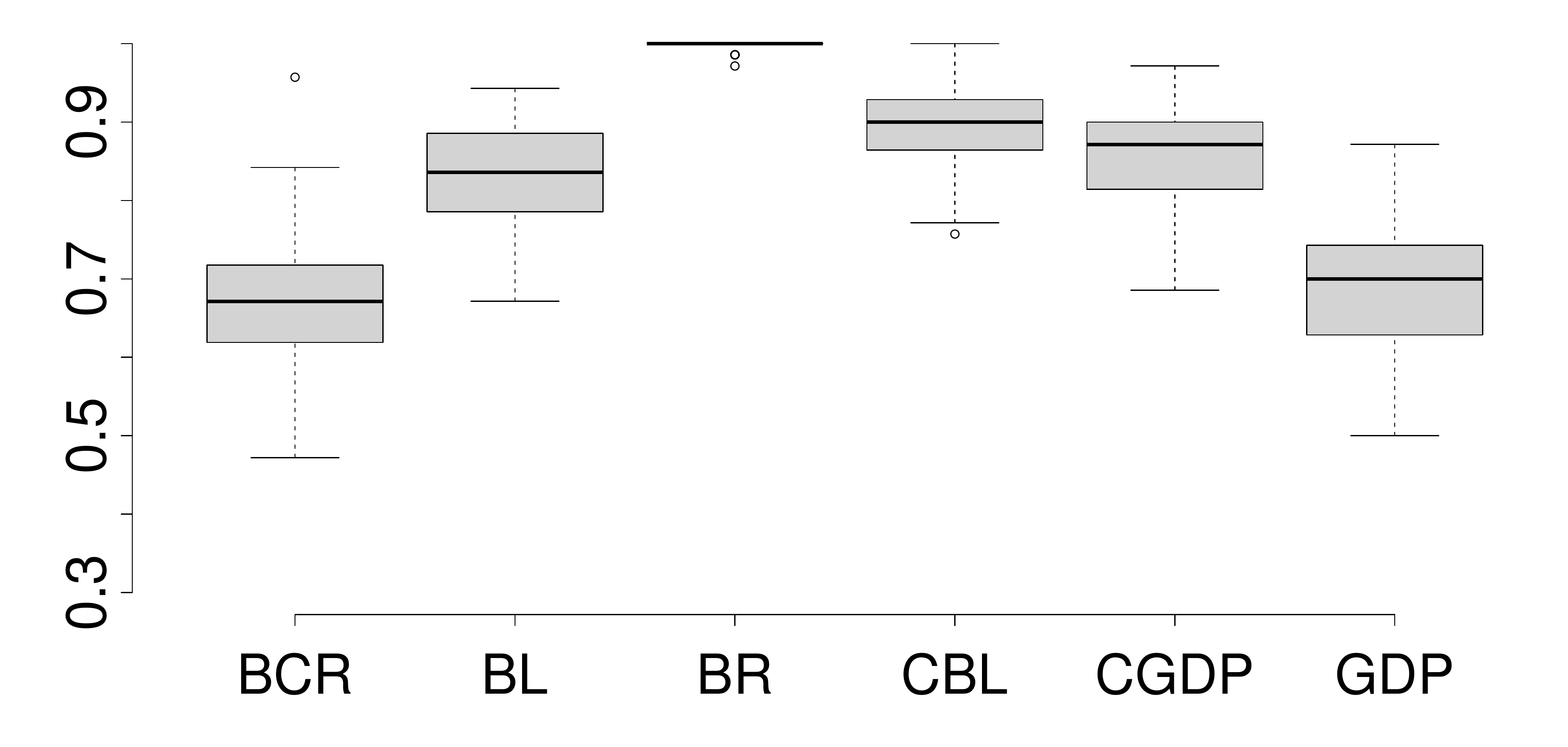}\label{syn-y0}}
    \subfigure[\emph{Model 2}]{\includegraphics[width=8cm]{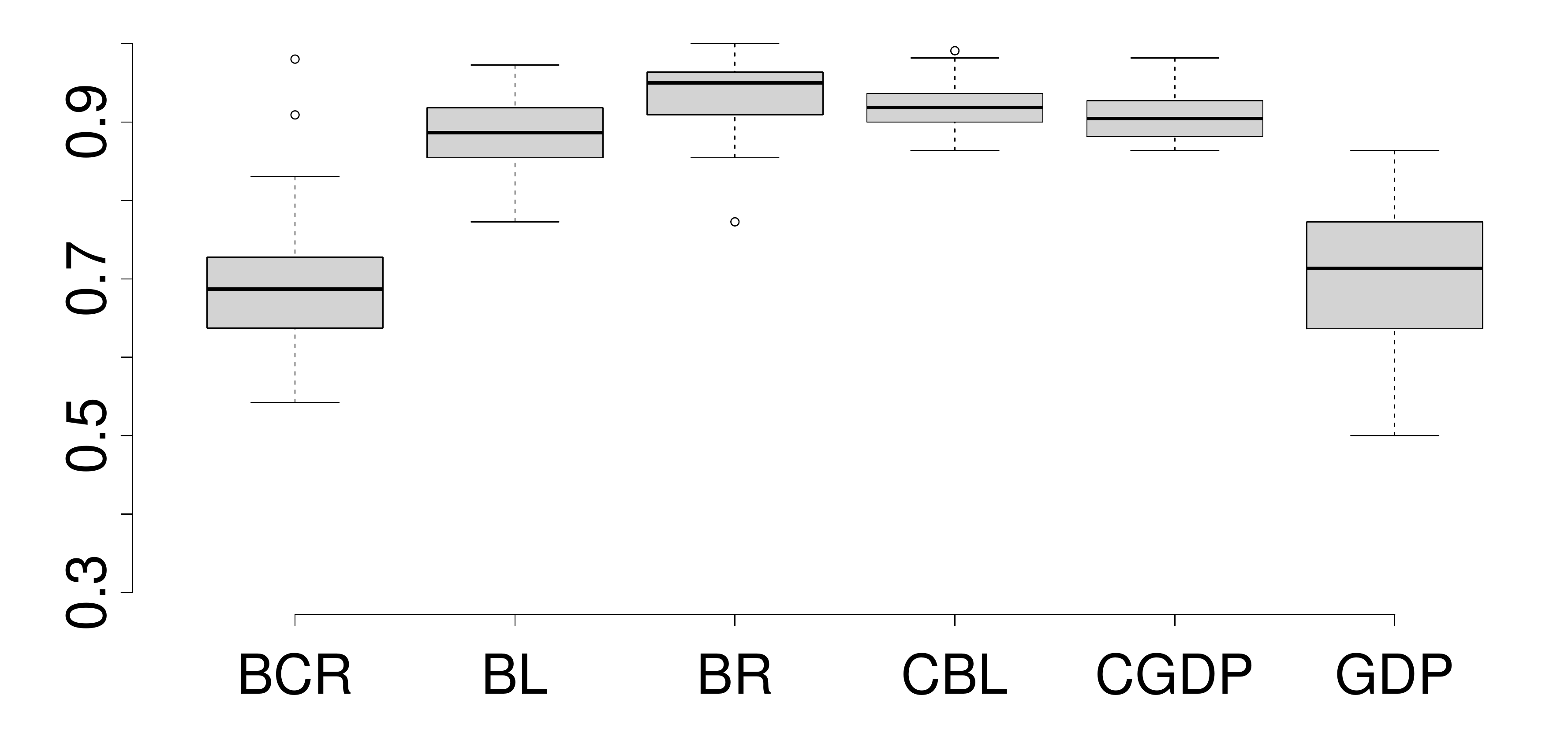}\label{syn-y1}}\\
    \subfigure[\emph{Model 3}]{\includegraphics[width=8cm]{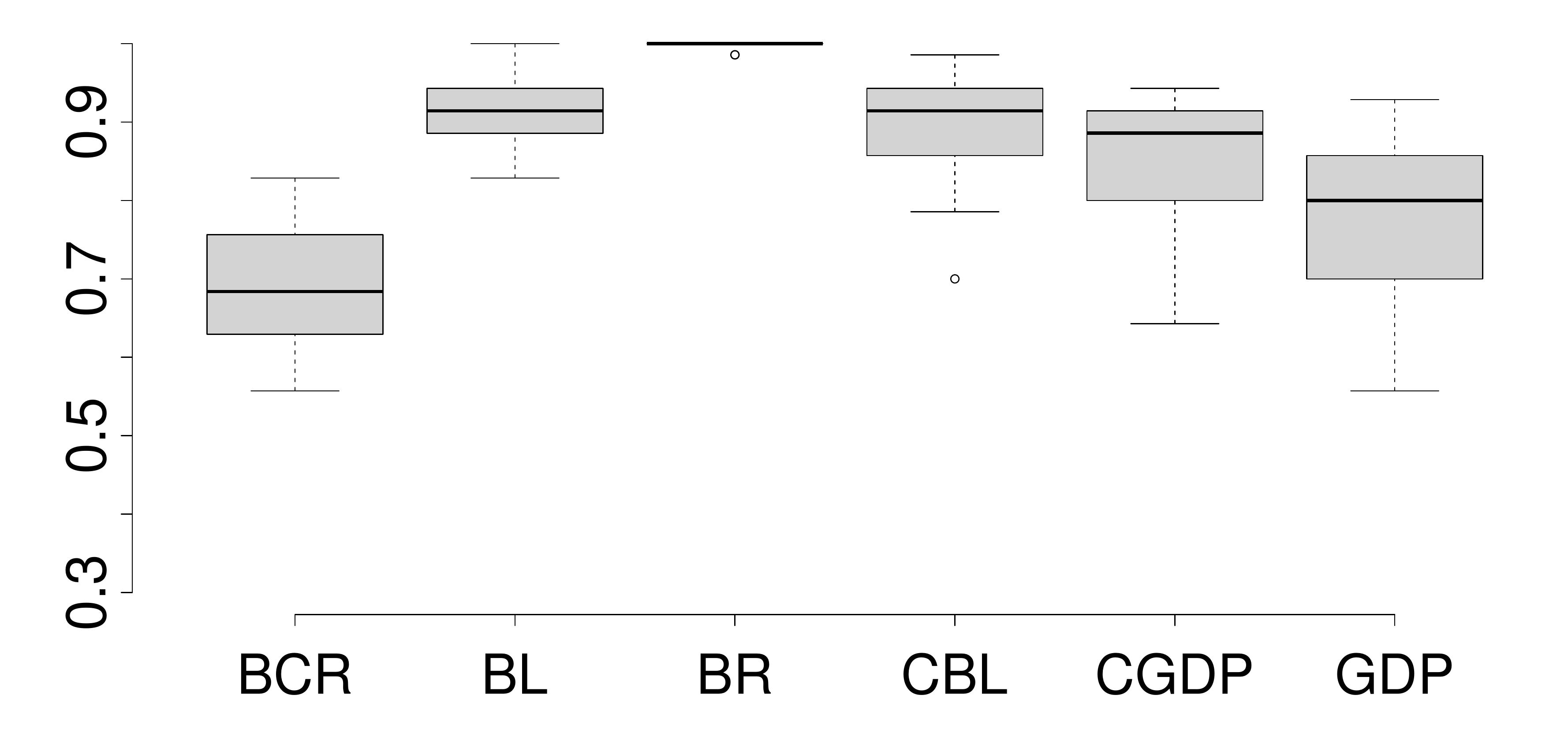}\label{syn-A10}}
    \subfigure[\emph{Model 4}]{\includegraphics[width=8cm]{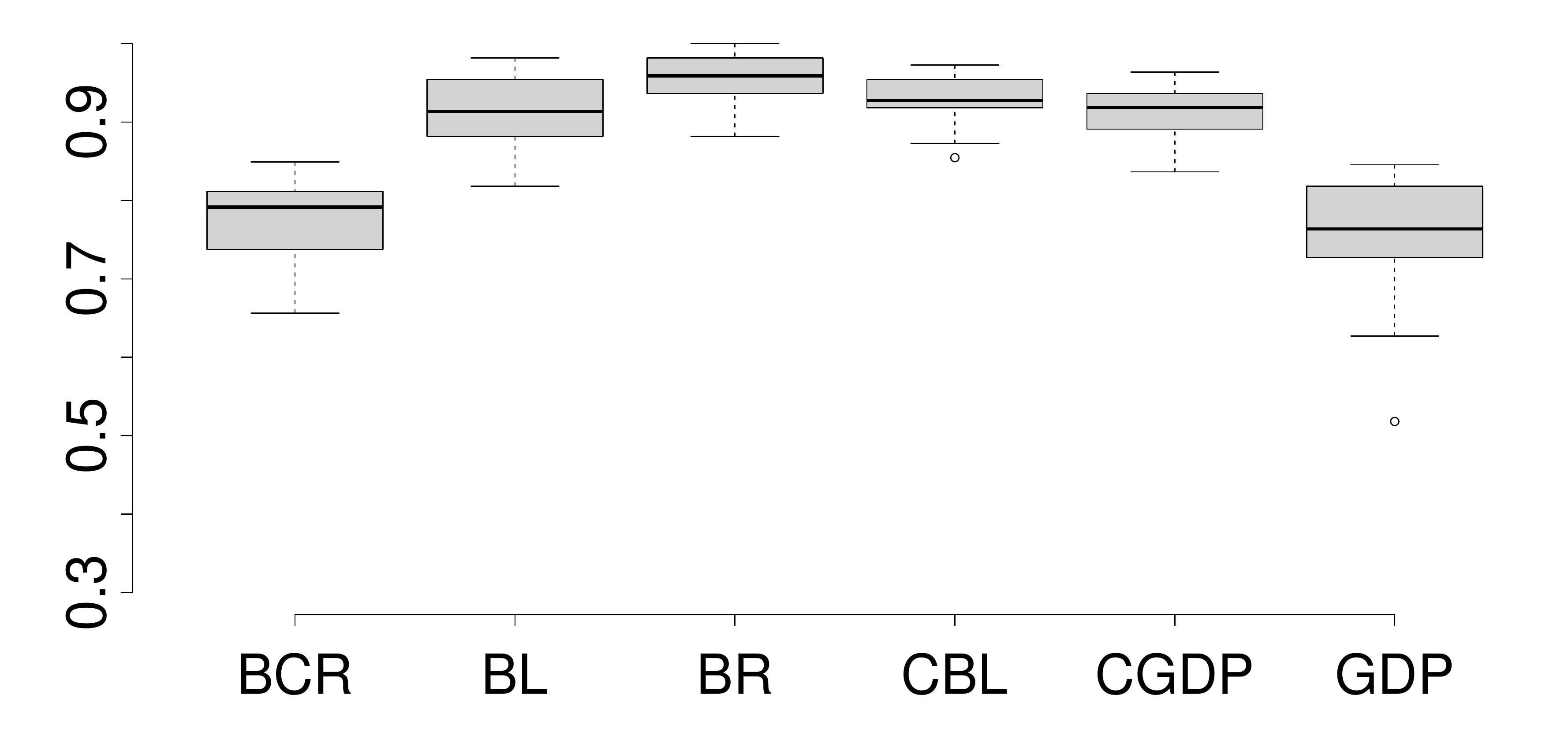}\label{syn-A11}}\\
    \subfigure[\emph{Model 5}]{\includegraphics[width=8cm]{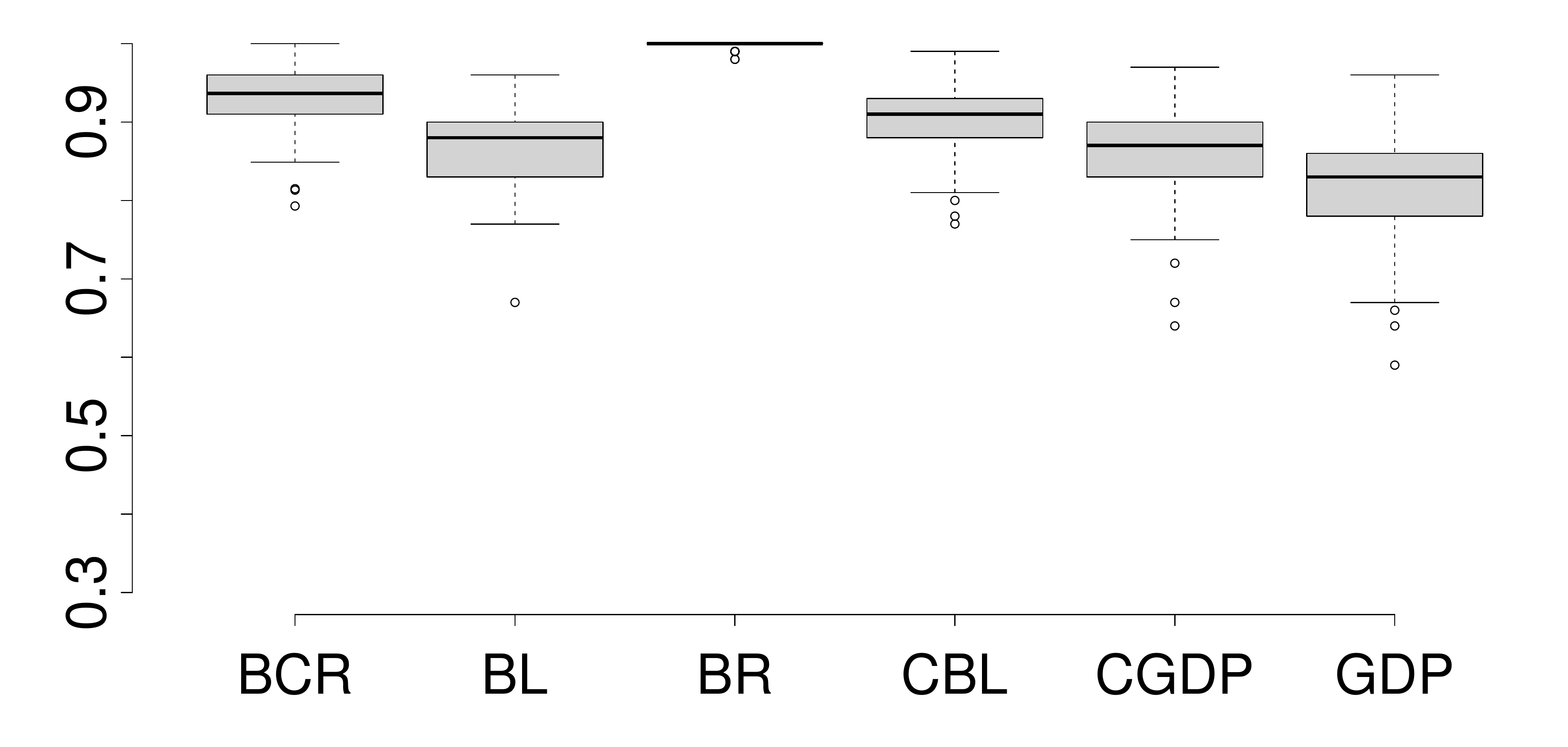}\label{syn-A00}}
    \subfigure[\emph{Model 6}]{\includegraphics[width=8cm]{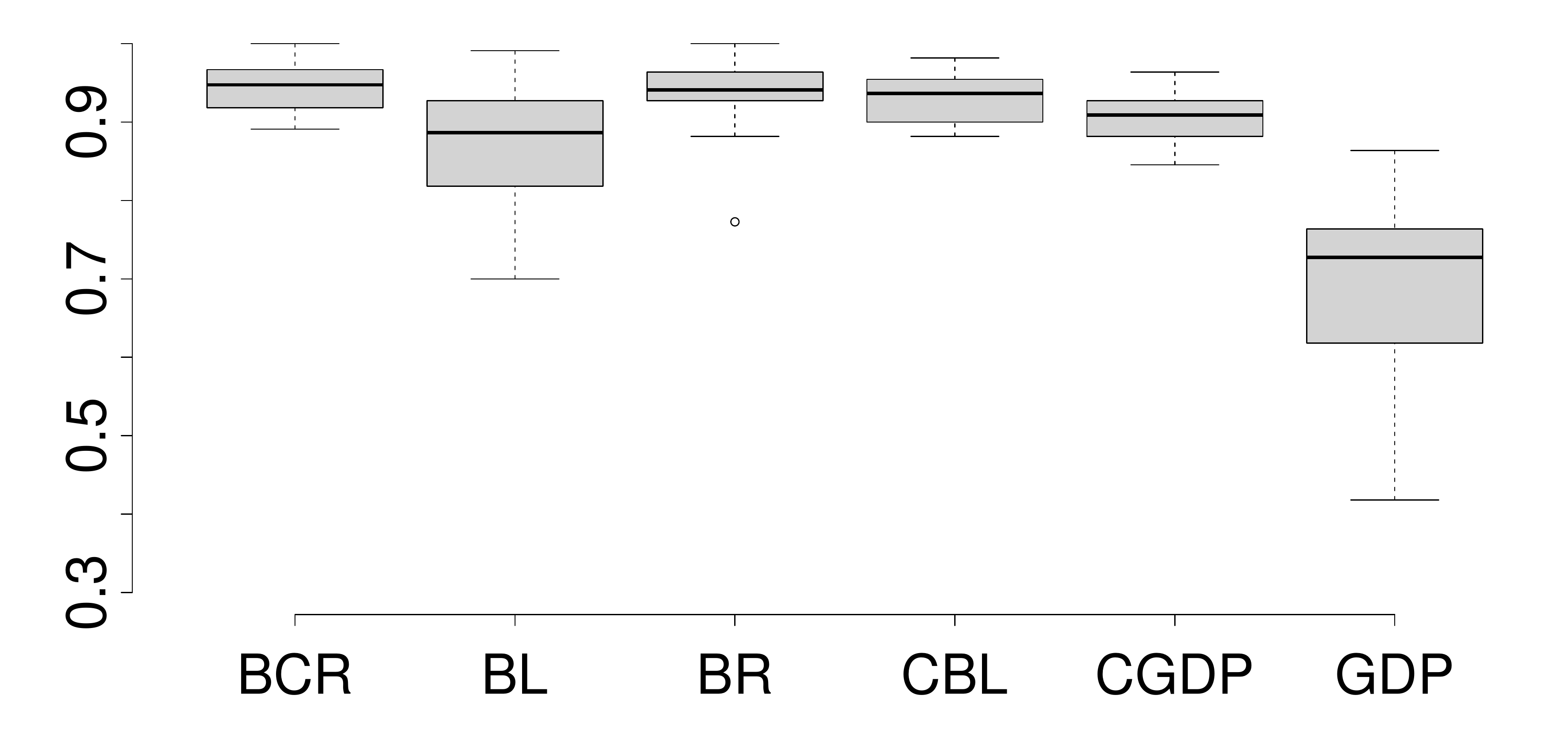}\label{syn-rho01}}
  \end{center}
\caption{Empirical coverage probability of 95\% predictive intervals for all the competing models.}\label{syn-surfs}
\end{figure}

{\noindent \emph{High dimension cases}}\\
To assess performance in much higher dimensional cases, we conduct a second set of simulation studies with $n=110$ and $p=\{15,000, 20,000, 25,000\}$. The MCMC-based BL, GDP and BR implementations become increasingly prohibitive as $p$ increases unless compression is employed. Hence, we use compressed versions of Bayesian Lasso (CBL), Generalized double Pareto (CGDP) and Bridge regression (CBR) as competitors. The following two data generating models are considered.\\
\newline
\emph{Model 1}: First 5 regression coefficients are 1, the rest are zero and $\sigma^2=1$.\\
\emph{Model 2}: All the regression coefficients are .1 and $\sigma^2=1$.\\
\newline

In each scenario, we simulate 100
data sets. Table~\ref{tab2} presents the MSPE averaged over these simulated datasets where in each dataset MSPE is calculated over 110 held-out observations. The values in the subscript represents bootstrap standard errors for the averaged MSPEs.
Table~\ref{tab2} shows the mean squared out-of-sample prediction errors along with their bootstrap standard errors for each of these models with two different sparsity levels.

\begin{table}[!th]
{\scriptsize
\begin{center}
\caption{\scriptsize{MSPE$\times .1$  for the competing models along with their bootstrap sd $\times .1$}}\label{tab2}
\begin{tabular}
[c]{ccccccc}%
\hline
&\multicolumn{3}{c}{Sparsity level 5} &\multicolumn{3}{c}{dense model}\\
\cline{2-7}
(n,p) & (110,15000) & (110,20000) & (110,25000) & (110,15000) & (110,20000) & (110,25000)\\
\cline{2-7}
          & & & & & &\\
BCR       & $0.82_{0.020}$ & $0.78_{0.011}$ & $0.80_{0.024}$ & $0.24_{0.01}$ & $0.28_{0.012}$  & $0.32_{0.01}$\\
          & & & & & &\\
CGDP      & $0.84_{0.013}$ & $0.86_{0.013}$ & $0.85_{0.015}$ & $3.23_{0.45}$ & $3.73_{0.59}$ & $5.99_{0.92}$\\
          & & & & & &\\
CBL      & $0.77_{0.011}$ & $0.79_{0.011}$ & $0.77_{0.012}$ & $3.12_{0.42}$ & $3.59_{0.55}$ & $5.70_{0.63}$\\
          & & & & & &\\
CBR       & $0.60_{0.007}$ & $0.61_{0.008}$ & $0.60_{0.007}$ & $2.93_{0.39}$ & $3.32_{0.43}$ & $5.21_{0.76}$\\
          & & & & & &\\
LASSO     & $0.22_{0.003}$ & $0.21_{0.003}$ & $0.19_{0.002}$ & $16.34_{0.18}$ & $21.98_{0.24}$ & $25.34_{0.38}$\\
          & & & & & &\\
RR        & $0.58_{0.006}$ & $0.59_{0.008}$ & $0.59_{0.007}$ & $15.12_{0.18}$ & $19.89_{0.24}$ & $24.95_{0.33}$\\
          & & & & & &\\
PLSR      & $0.58_{0.006}$ & $0.59_{0.008}$ & $0.59_{0.007}$ & $15.12_{0.19}$ & $19.89_{0.23}$ & $24.95_{0.32}$\\
\hline
\end{tabular}
\end{center}
}
\end{table}
If the true model is sparse, LASSO, RR and CBR are the three best performing methods with BCR showing competitive performance. As sparsity decreases, all the sparsity favoring models perform poorly.  In the dense case, all the compressed models perform significantly better than the corresponding sparsity favoring models. BCR, in particular, shows excellent performance in this scenario.  Figure~\ref{par2} shows boxplots for the empirical coverage probabilities of 95\% PIs for all the Bayesian models. BCR has coverage probabilities between 90-98\% in all cases. Other compressed regression models are also found to deliver excellent coverage probabilities.
We also compute the coverage probabilities of LASSO, RR and PLSR using the two stage plug-in approach. When the sparsity is low, the median coverage probabilities (with 95\% CIs) of LASSO are $.75 (.27, .85)$, $.63 (.31, .84)$ and $.47 (.02,.79)$ for $p=15,000, 20,000, 25,000$ respectively. For the dense cases, the coverage probabilities of LASSO are  $.93 (.22, .97)$, $.90 (.19, .95)$ and $.93 (.13,.96)$ in $p=15,000, 20,000, 25,000$ respectively. RR and PLSR are found to suffer from severe under-coverage.

\begin{figure}[!ht]
  \begin{center}
    \subfigure[$p=15000$, sparsity 5]{\includegraphics[width=8cm]{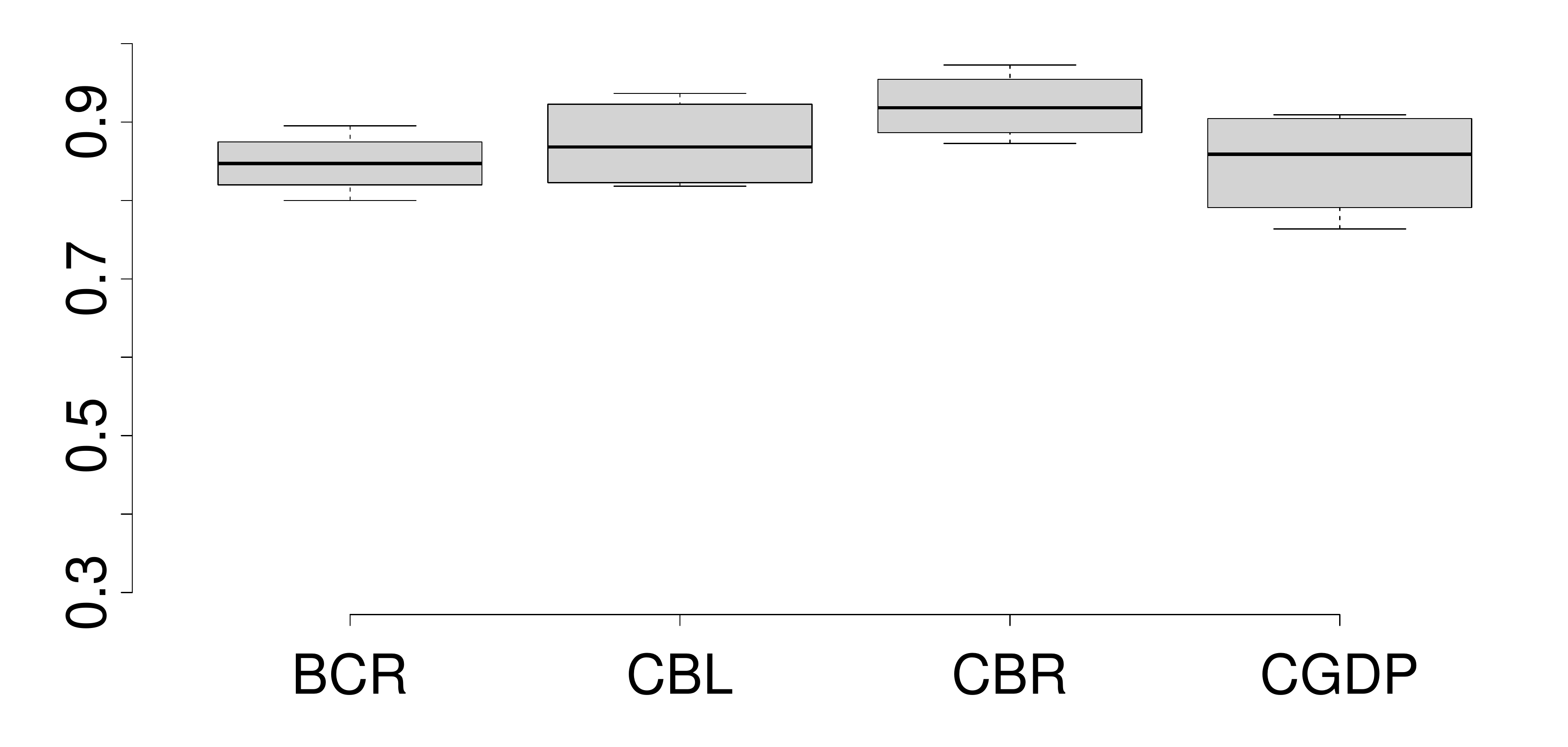}\label{sp1}}
    \subfigure[$p=15000$, dense]{\includegraphics[width=8cm]{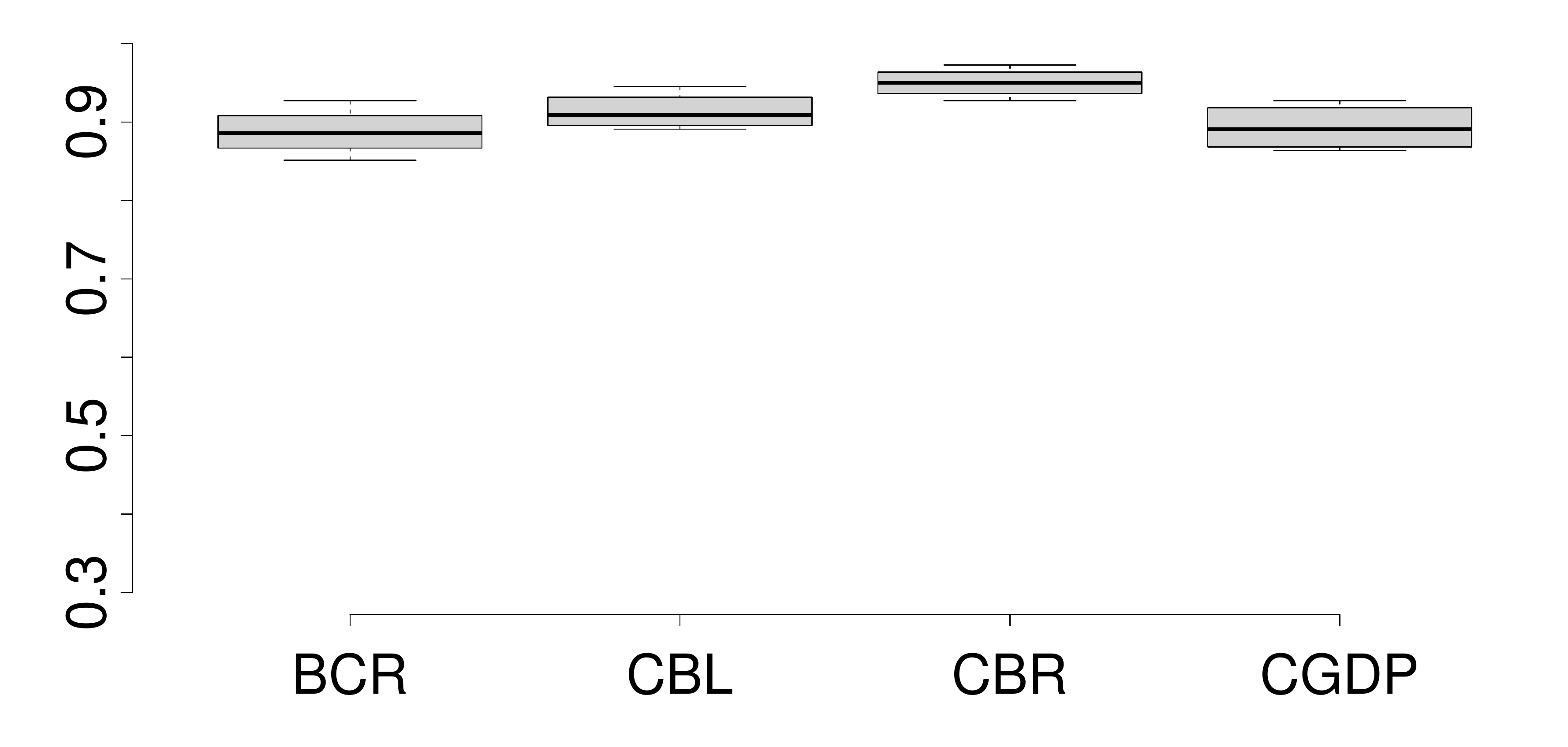}\label{dp1}}\\
    \subfigure[$p=20000$, sparsity 5]{\includegraphics[width=8cm]{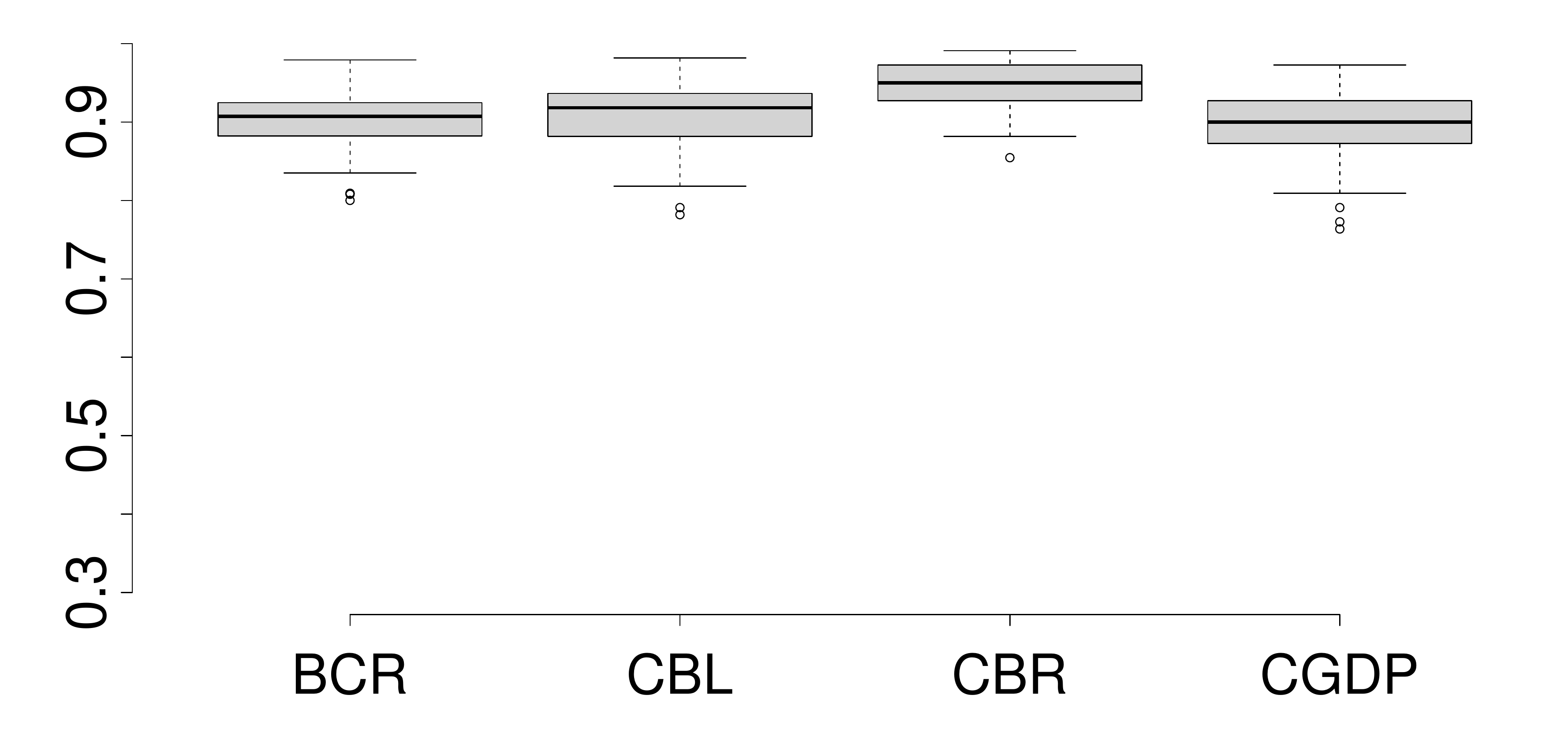}\label{sp2}}
    \subfigure[$p=20000$, dense]{\includegraphics[width=8cm]{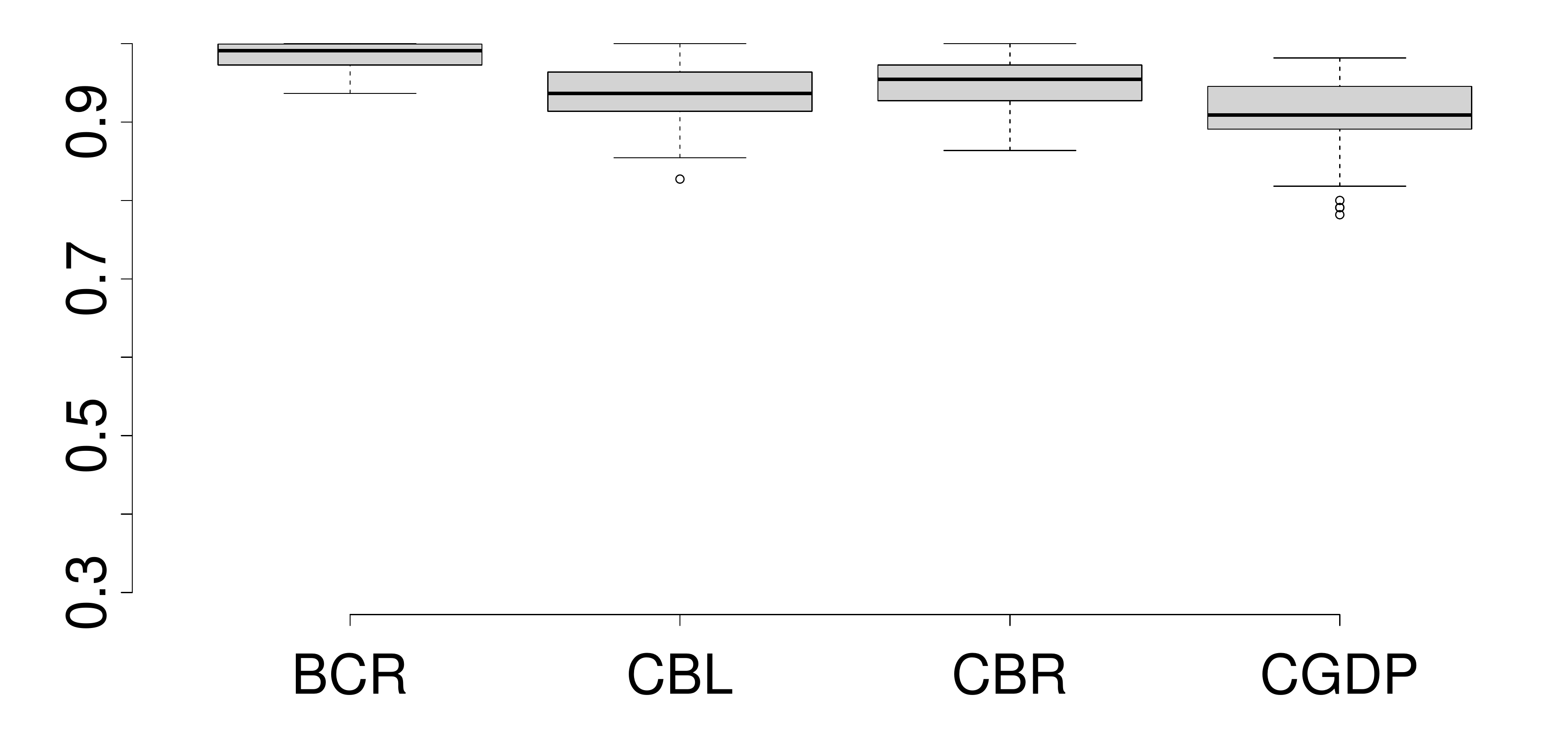}\label{dp2}}\\
    \subfigure[$p=25000$, sparsity 5]{\includegraphics[width=8cm]{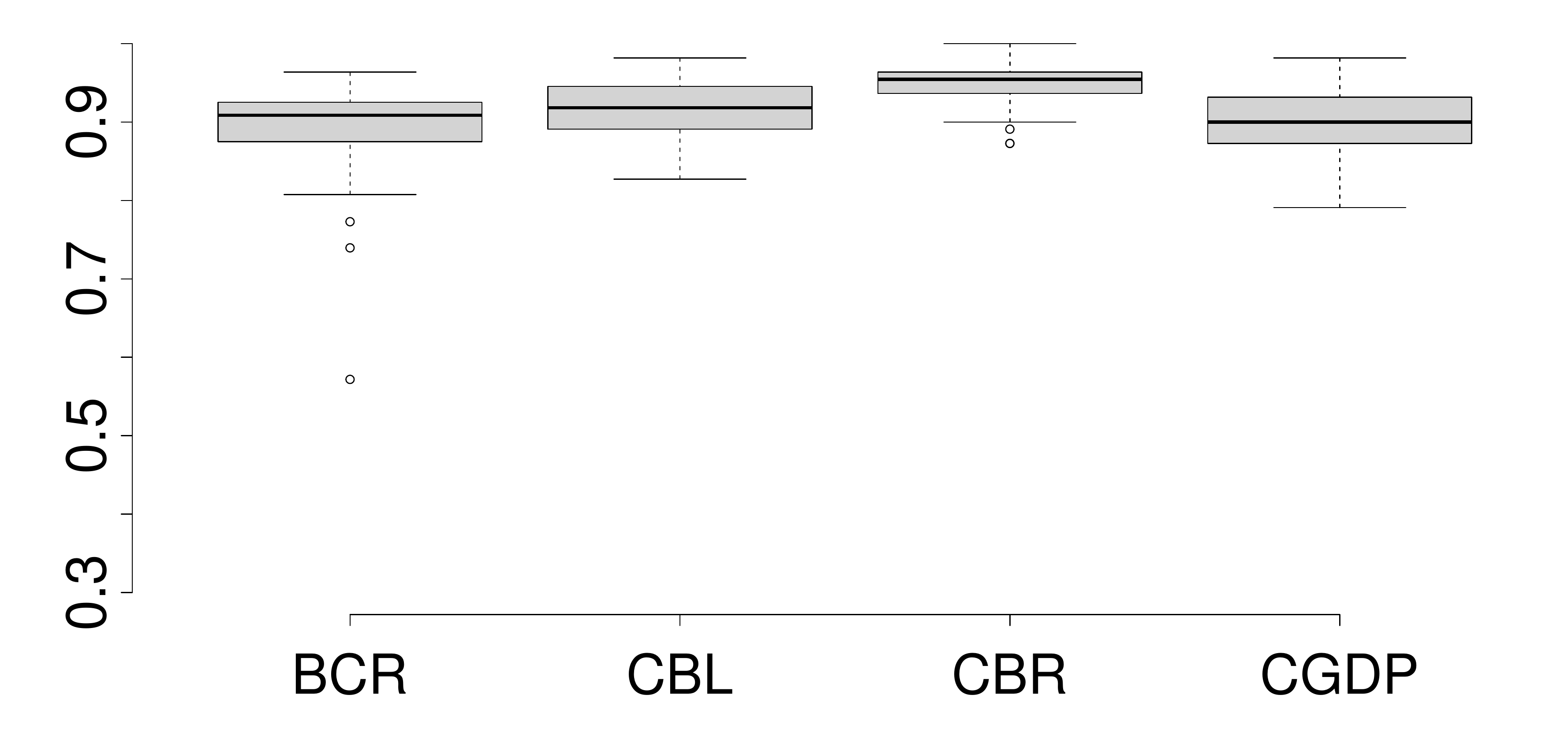}\label{sp3}}
    \subfigure[$p=25000$, dense]{\includegraphics[width=8cm]{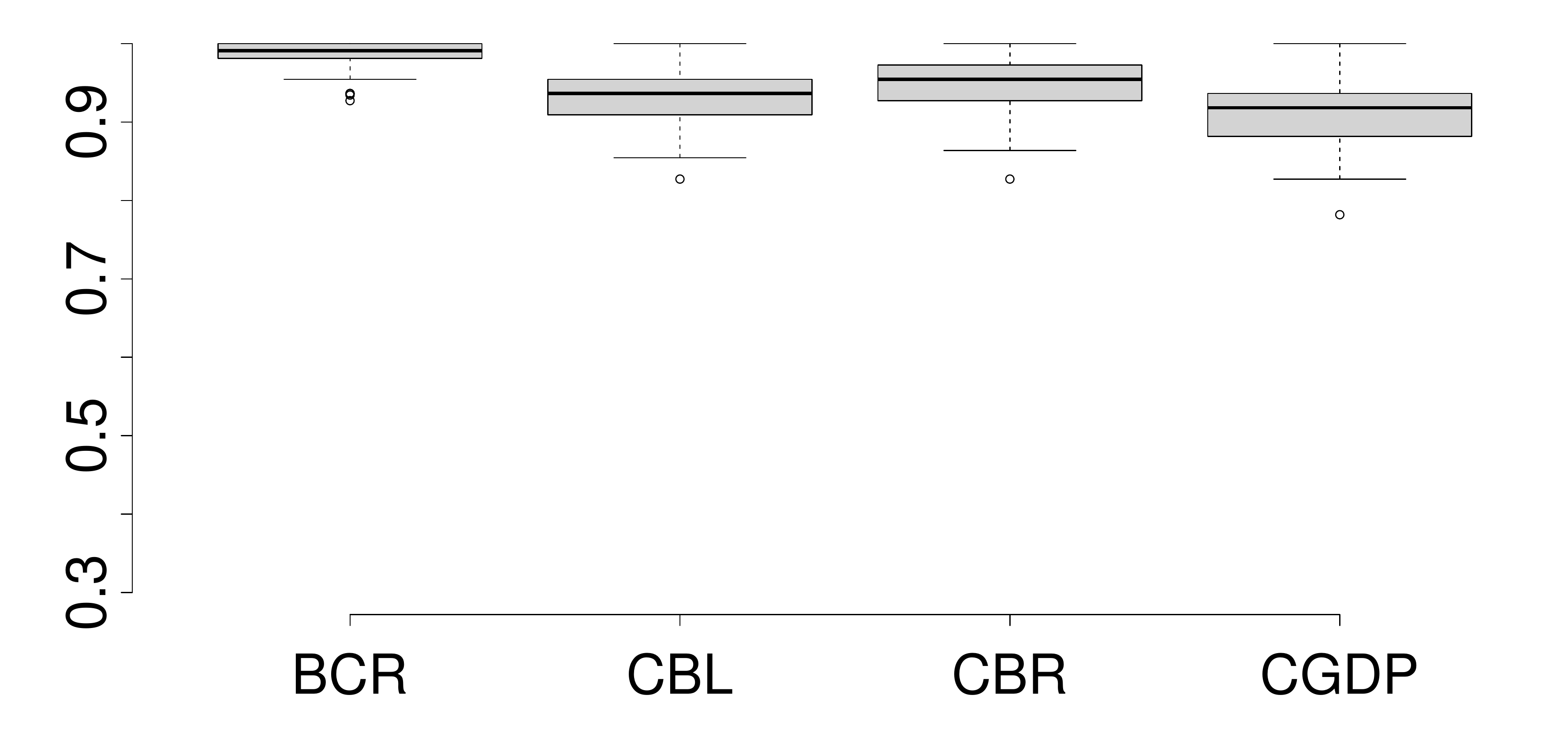}\label{dp3}}
  \end{center}
\caption{Empirical coverage probability of 95\% predictive intervals for all the competing models.}\label{par2}
\end{figure}

{\noindent \emph{Computational Speed}}\\
A crucial consideration is comparing alternative methods is computing time.  The approach of applying MCMC to the compressed data, which is employed in CBL, CGDP and CBR, is reasonably fast to implement.  Using non-optimized R code implemented on a single 3.06-GHz
Intel Xeon processor with 4.0 Gbytes of random-access memory running Debian LINUX, the computing time for 10,000 iterations of CGDP and CBL in the $n=110$ and $p=25,000$ cases was only 23.9 seconds.

One of the major advantages of BCR is the rapid (often essentially instantaneous) computing time. Important contributors to computation time include data compression, which involves Gram-Schmidt orthogonalization of $m$ rows of an $m\times p$ matrix,  multiplying an $n \times p$ and $p \times m$ matrix, as well as the time to invert low dimensional $m \times m$ matrices in the process of calculating the posterior and posterior weights.  If the sample size $n$ is not large, these inversions are very quick.  Only the matrix multiplies and the Gram-Schmidt orthogonalization involved in the compression convey increasing burden with increasing $p$.  Given that the whole model averaging process is embarrassingly parallelizable over different choices of $m$, the computation can be done very quickly using a parallel implementation with sufficiently many processors. Even not exploiting any parallelization, one obtains results quickly using non-optimized R code on a single 3.06-GHz Intel Xeon processor with 4.0 Gbytes of random-access memory running Debian LINUX.

Figure~\ref{timecomp} shows the computational speed comparison between BCR and LASSO for $n=110$. Figure~\ref{timecomp} indicates that when $p$ is small, BCR enjoys little computational advantage over LASSO. The advantage is substantial as $p$ increases; this becomes particularly notable as we scale from tens of thousands of predictors to millions or more, which is becoming increasingly common.

\begin{figure}[!ht]
  \begin{center}
  \includegraphics[width=16cm]{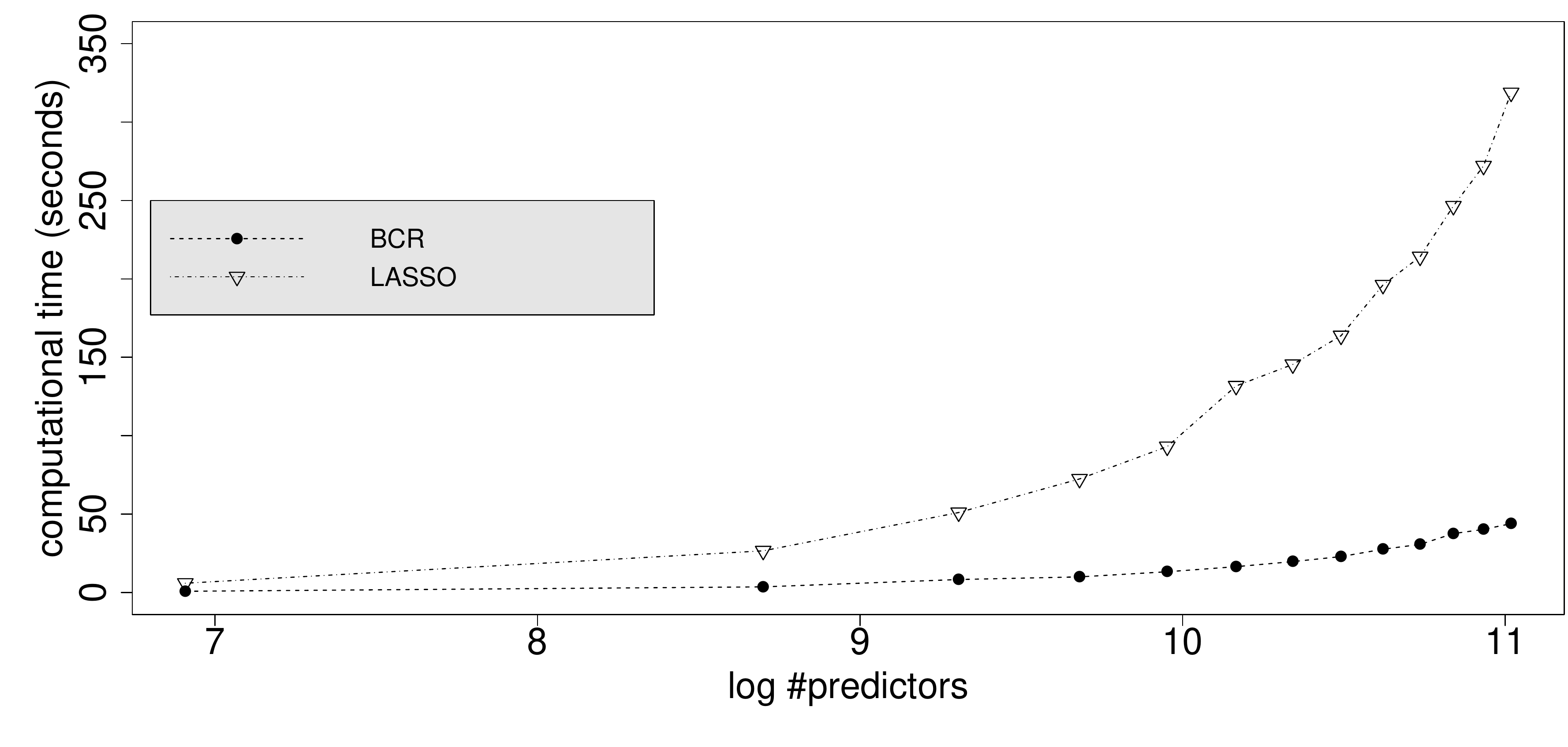}
  \end{center}
\caption{Computational time in seconds for LASSO and BCR against log of the number of predictors.}\label{timecomp}
\end{figure}
These computational comparisons are using a parallel implementation of BCR. This can be further improved parallelizing the matrix multiplies involved in the initial data compression and Gram-Schmidt orthogonalization.

\section{Molecular Epidemiological Study}

We apply Bayesian compressed regression to data from a molecular epidemiology study. The focus  is on assessing gene-environment interactions between chemical exposures and  variants in genes thought to be important in DNA damage and repair pathways.  Immortalized cell lines were established for 90 individuals chosen to represent the ethnic mix in the United States.  For these individuals, single nucleotide polymorphisms (SNPs)  in  DNA damage and repair genes were measured; in particular, there were 42,181 SNPs discarding loci for which there is no variability among individuals in the sample.  Replicated samples over 100 cells from each
cell line were allocated to 1 of 3 groups: (i) analyzed without treatment, (ii) analyzed immediately after exposure to a known genotoxic agent (MMS, $H_2O_2$), or (iii) analyzed after allowing some time (10, 15, 240 minutes) for DNA repair. We will code these groups as ``NT", ``0'' and ``Later" respectively.

The frequency of DNA strand breaks was measured for each cell using single cell gel electrophoresis, which is also known as the comet assay.  When subject to electrophoresis, the nucleoid of cells with many strand breaks pull apart, changing in shape from a ball to resemble a comet.  Comet assay image processing software produces multiple measures of the amount of DNA in the comet tail, which provides a useful surrogate.  The Olive tail moment (Olive et al., 1990) has been established as the best of the surrogates in previous studies (Dunson et al., 2003), and is defined as  the percentage of DNA in the tail of
the comet multiplied by the length between the center of the comet's head and tail.

For each cell line for each group and for each genotoxic agent (MMS, $H_2O_2$), we compute 33 quantiles of the Olive tail moments. For cell line $i$, group
$j$ ($j=NT,0,Later$) and agent $k$ ($k=MMS,H_2O_2$), $y_{i,j,k}=(y_{i,j,k,1},...,y_{i,j,k,33})'$ is the vector of 33 quantiles of Olive tail moments. We derive two sets of response variables from the $y_{i,j,k}$'s. Let
\begin{enumerate}
\item $O_{i,k}=y_{i,0,k}-y_{i,NT,k}=(O_{i,k,1},...,O_{i,k,33})'=$ change in DNA damage from ``NT" to 0.
\item $J_{i,k}=y_{i,0,k}-y_{i,Later,k}=(J_{i,k,1},...,J_{i,k,33})'=$ change in DNA damage from 0 to ``Later".
\end{enumerate}
While $O_{i,k}$'s measure an individual's sensitivity to DNA damage induced by toxic reagents, $J_{i,k}$'s quantify an individual's repair rate after damage.
Predictors include the SNPs ($x_i$) and the type of genotoxic agent, $E_{i,k,l} = 1$ for $H_2O_2$ and $0$ for MMS.  We consider the following linear
regression models:
\begin{equation}
O_{i,k,l}=\alpha E_{i,k,l}+x_i'\gamma_1+E_{i,k,l}x_i'\gamma_2+\epsilon_{i,k,l},\:\epsilon_{i,k,l}\sim N(0,\sigma^2).
\end{equation}
A similar set of $33$ independent models have been fitted using the responses $J_{i,k,l}$s.  Each of these models includes main effects for the SNPs and interactions between the SNPs and exposure type, leading to $p=84,363$ predictors.

Analyses are conducted as in the simulation section, with 10-fold cross validation used.  We implemented all competing methods for each of the
 33 quantiles for the two sets of responses, but  the standard package used to implement LASSO (lars) failed to converge even when we modified tuning parameters, seemingly due to the dimensionality and discrete nature of the predictors.  Hence, we do not report results for LASSO.  To get LASSO to run, we used a randomly selected subsample of 30,000 SNPs, and in those analyses we found substantially improved out of sample predictive performance for our compressed regression methods compared with LASSO and the other competitors uniformly across quantiles.  We do not present those results here but focus on analysis of the complete data set.

Below we present squared correlation between the
observed and the fitted responses for some of the initial quantiles. This is an indicator of the amount of variability in the response explained by the predictors. It is evident both from Table~\ref{tab5} and Table~\ref{tab6} that all the compressed models show high correlation between observed and predicted responses, although compressed Bridge regression (CBR) performs better than the other competitors uniformly. PLS performs similar to the BCR and RR shows worse performance than BCR. The figures in the brackets present the length of the 95\% PI's for the competing methods. We fit responses corresponding to the higher quantiles and found decreasing squared correlation between observed and fitted responses.

\begin{figure}[h]
  \begin{center}
    \subfigure[quantiles for $\bO_{i,k}$]{\includegraphics[width=12cm]{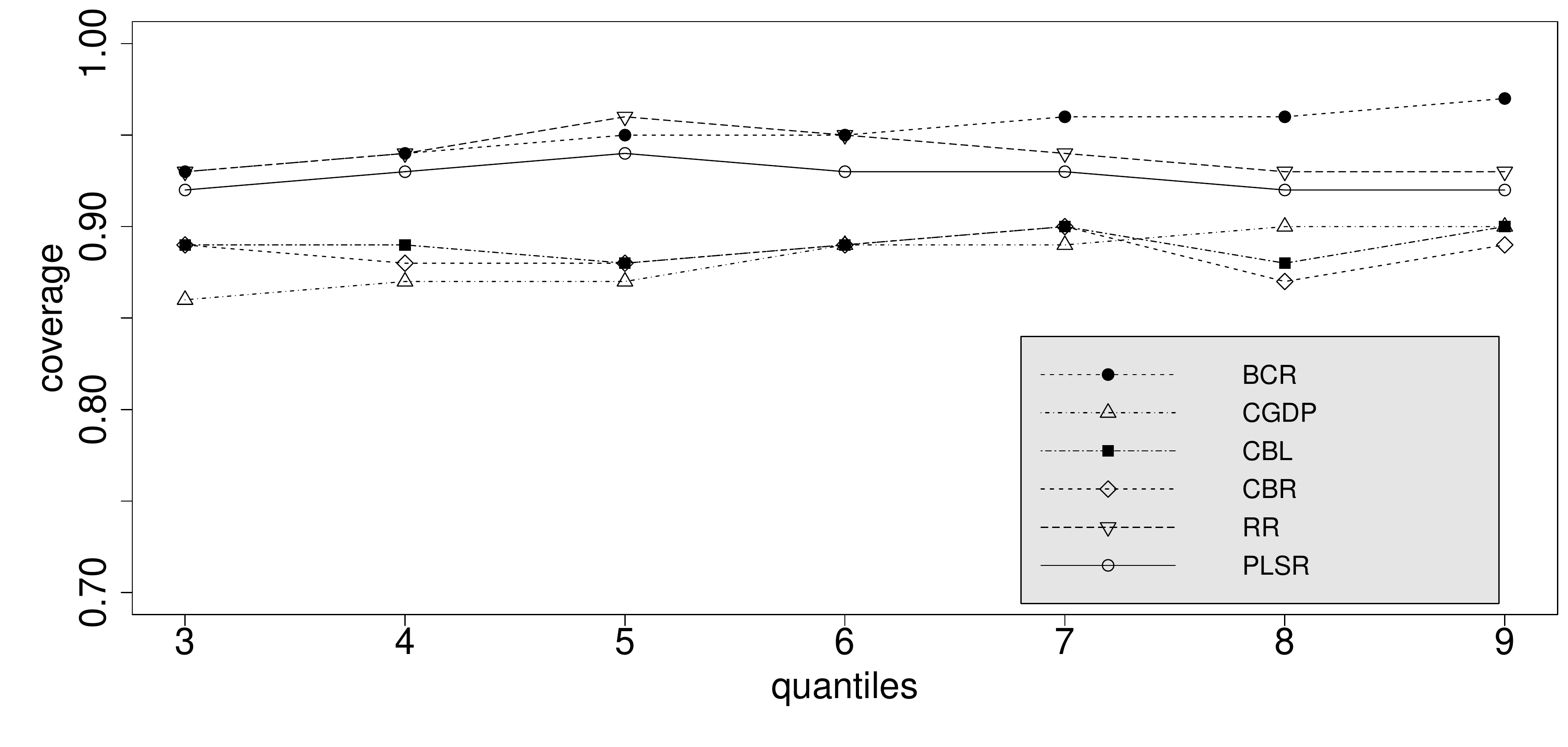}\label{resp1}}\\
   \subfigure[quantiles for $\bJ_{i,k}$]{\includegraphics[width=12cm]{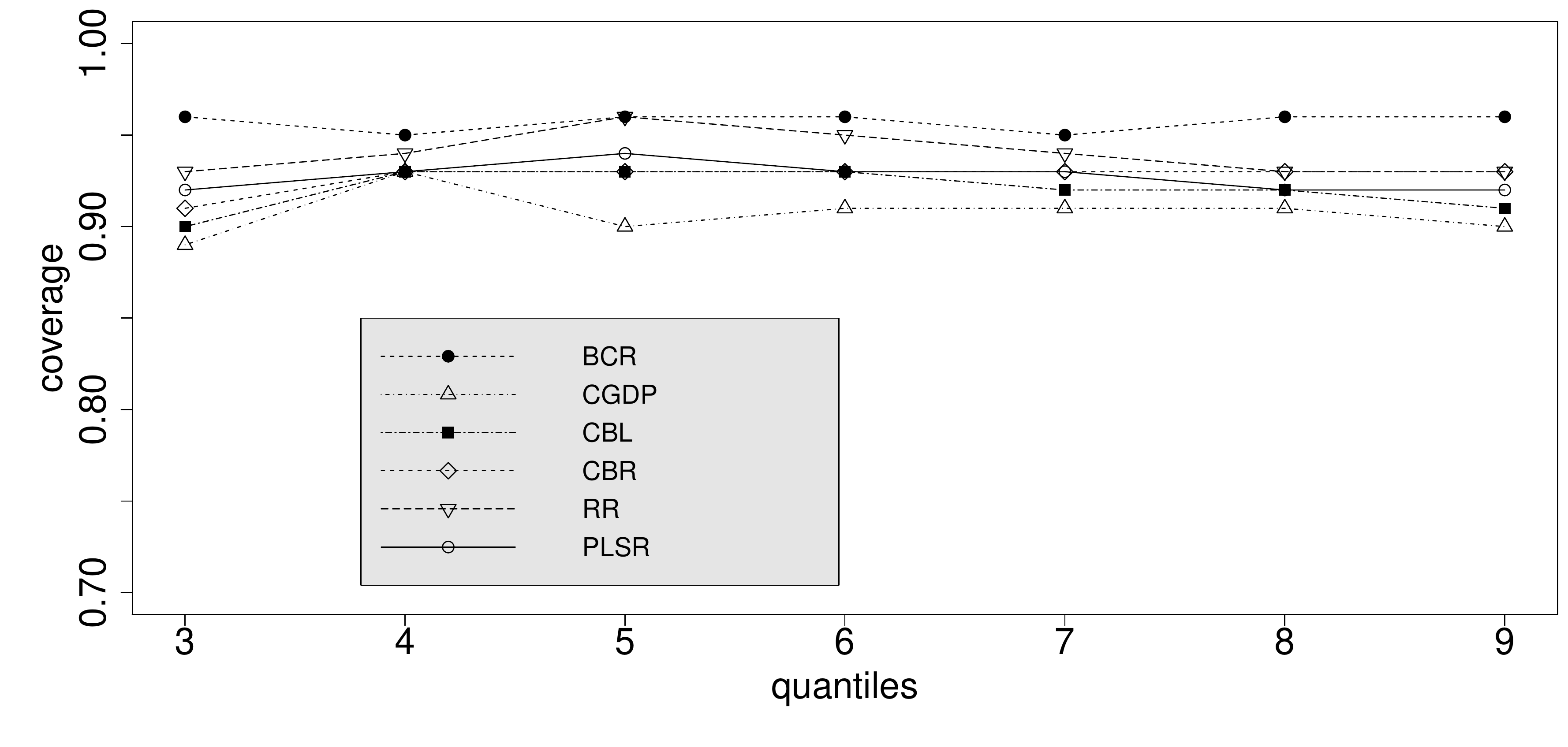}\label{resp2}}
  \end{center}
\caption{Coverage of 95\% PI's for the response variables}\label{coverquant}
\end{figure}

Figure~\ref{coverquant}
represents the coverage of 95\% PI's for BCR, CGDP, CBL and CBR, PLSR and RR for the two types of responses. It is evident from the plot that while
BCR shows excellent coverage, other compressed approaches have a slight under coverage.
This can be attributed to the fact that the lengths of the 95\% PI's for CGDP, CBL, CBR are much narrower than the lengths for BCR as presented in Table~\ref{tab5} and Table~\ref{tab6}. The frequentist methods also suffer from some under-coverage for the second set of responses, although they show competitive performance with BCR for the first set of responses.
\begin{table}[!th]
{\scriptsize
\begin{center}
\caption{Squared correlation between the predicted and observed responses for $\bO_{i,k}$.The figures in the
bracket represent length of the 95\% PI's for compressed models.}\label{tab5}
\begin{tabular}
[c]{cccccccccc}%
\hline
responses & resp3 & resp4 & resp5 & resp6 & resp7 & resp8 & resp9\\
\hline
           & & & & & & &\\
BCR        & 0.55 (15.45) & 0.54 (16.11)& 0.54 (16.52)& 0.50 (16.89)& 0.48 (17.49)& 0.45 (18.02)& 0.43 (18.34)\\
           & & & & & & &\\
CGDP       & 0.65 (7.51) & 0.63 (7.84)& 0.61 (8.14)& 0.59 (8.34)& 0.57 (8.63)& 0.55 (8.90)& 0.52 (9.05)\\
           & & & & & & &\\
CBL        & 0.70 (7.39)& 0.68 (7.71)& 0.67 (7.92)& 0.65 (8.12)& 0.63 (8.40)& 0.61 (8.64)& 0.59 (8.79)\\
           & & & & & & &\\
CBR        & 0.73 (7.28) & 0.71 (7.62)& 0.70 (7.81)& 0.69 (8.01)& 0.67 (8.29)& 0.65 (8.53)& 0.63 (8.69)\\
           & & & & & & &\\
LASSO       &  (--) &  (--) &  (--)&   (--) & (--)  & (--) & (--)\\
           & & & & & & &\\
RR         & 0.48 (11.26)& 0.44 (11.87)& 0.42 (12.19)& 0.40 (12.46)& 0.37 (12.99)& 0.33 (13.47)& 0.31 (13.71)\\
           & & & & & & &\\
PLSR       & 0.54 (11.03)& 0.50 (11.60)& 0.49 (11.73)& 0.48 (11.83)& 0.44 (12.46)& 0.40 (12.92)& 0.38 (13.05)\\
\hline
\end{tabular}
\end{center}
}
\end{table}

\begin{table}[!th]
{\scriptsize
\begin{center}
\caption{Squared correlation between the predicted and observed responses for $\bJ_{i,k}$. The figures in the
bracket represent length of the 95\% PI's for compressed models.}\label{tab6}
\begin{tabular}
[c]{cccccccc}%
\hline
responses & resp3 & resp4 & resp5 & resp6 & resp7 & resp8 & resp9\\
\hline
           & & & & & & &\\
BCR        & 0.46 (11.53)& 0.43 (11.83)& 0.39 (11.95)& 0.34 (12.11)& 0.31 (12.25)& 0.27 (12.68)& 0.22 (12.99)\\
           & & & & & & &\\
CGDP       & 0.55 (5.76)& 0.50 (5.95)& 0.46 (6.09)& 0.45 (6.12)& 0.42 (6.17)& 0.38 (6.32)& 0.35 (6.43)\\
           & & & & & & &\\
CBL        & 0.62 (5.59)& 0.58 (5.75)& 0.56 (5.84)& 0.53 (5.89)& 0.51 (5.95)& 0.46 (6.11)& 0.42 (6.22)\\
           & & & & & & &\\
CBR        & 0.66 (5.48)& 0.62 (5.66)& 0.60 (5.73)& 0.57 (5.80)& 0.55 (5.88)& 0.51 (6.04)& 0.47 (6.16)\\
           & & & & & & &\\
LASSO      & (--) & (--)& (--) & (--) & (--) & (--) & (--) \\
           & & & & & & &\\
RR         & 0.38 (8.07)& 0.33 (8.41)& 0.29 (8.57)& 0.26 (8.63)& 0.21 (8.88)& 0.16 (9.24)& 0.13 (9.45)\\
           & & & & & & &\\
PLSR       & 0.51 (7.03)& 0.46 (7.30)& 0.42 (7.45)& 0.38 (7.55)& 0.34 (7.73)& 0.28 (8.03)& 0.24 (8.14)\\
\hline
\end{tabular}
\end{center}
}
\end{table}

\section{Discussion}

The overarching goal of the proposed approach is to define a practical approximation to Bayesian inference in massive dimensional regression settings through the use of random compression of the predictors prior to analysis.  Given the dramatic computational gains, we expect to pay some price in terms of predictive accuracy and were pleasantly surprised that this price seemed to be small in most cases.  In fact, the approach had better performance in terms of mean square prediction error overall than many standard competitors due in part to the ability to characterize sparse as well as dense cases in which predictors can be compressed to a lower-dimensional linear subspace with minimal loss of information about the response.

There are many natural directions in terms of future research.  Firstly, the random projection approach can be directly extended to cases involving matrix or tensor-valued predictors, which increasingly arise in modern applications.  For example, if the predictor $X_i$ for subject $i$ is matrix-valued, then one can pre- and post-multiply by random projection matrices to compress to a lower-dimensional matrix that will be easier to handle computationally.  Similarly, in large $n$ and $p$ settings one can compress the enormous $n \times p$ design matrix by pre- and post-multiplication to obtain a smaller data set, while also appropriately compressing the response.  This would combine compressed sensing and compressed regression.


Although our focus has been on Bayesian approaches in GLMs, compressed regression can similarly be used much more broadly.  For example, one can estimate frequentist regression models for many different randomly compressed predictor vectors, with the results averaged in performing predictions.  This approach is reminiscent of random forests, but is quite different in nature and could be referred to as compressed averaging (Caving).  Finally, although the dramatic computational gains are attributable to the use of random projections generated in advance of the analysis, there is the possibility of refining these projections based on the available data; ideally this could be done in a manner that only adds modestly to the computational burden.

\section*{References}

\begin{description}
\item Armagan, A., Dunson, D.B., and Lee, J. ``Generalized double pareto shrinkage." \emph{Statistica Sinica}, 23:119-143, 2012.

\item Berger, J., O. ``A robust generalized Bayes estimator and confidence region for a multivariate normal mean." \emph{Annals of Statistics},
8(4):716-761, 1980.

\item Bhattacharya, A., Pati, D., Pillai, N. and Dunson, D.B. ``Bayesian shrinkage." \emph{Arxiv Preprint arxiv}:1212.6088, 2012.

\item Bickel, P.J., Ritov, Y., and Tsybakov, A.B. ``Simultaneous analysis of Lasso and Dantzig selectors." \emph{Annals of Statistics},
37(4):1705-1732, 2009.

\item Carvalho, C.M., Polson, N.G., and Scott, J.G. ``Handling sparsity via the horseshoe." \emph{JMLR: W \& CP}, 5:73-80, 2009.

\item Carvalho, C.M., Polson, N.G., and Scott, J.G. ``The horseshoe estimator for sparse signals." \emph{Biometrika}, 97(2):465-480, 2010.

\item Candes, E.A., Romberg, J. and Tao, T. ``Stable signal recovery from incomplete and inaccurate measurements." \emph{Communications in Pure and Applied Mathematics}, 59(8):1207-1223, 2006.

\item Candes, E.A., and Tao, T. ``Decoding by linear programming." \emph{IEEE Trans. Info. Theory}, 51(12):4203-4215, 2005.

\item Candes, E.A., and Tao, T. ``The Dantzig selector: statistical estimator when p is much larger than n."
     \emph{Annals of Statistics}, 35(6): 2313-2351, 2007.

\item Cook, R.D. \textit{Regression graphics:ideas for studying regression through graphics}, New York:Wiley, 1998.

\item Donoho, D. ``Compressed sensing." \emph{IEEE Trans. Info. Theory}, 52(4):1289-1306, 2006.

\item DuMouchel, W. ``Data squashing: constructing summary data sets." \emph{www.cs.princeton.edu}, 1999.

\item Davenport, M., Durate, M., Wakin, M., Laska, J., Takhar, D., Kelly, K., and Baraniuk,
R. ``The smashed filter for compressive classification and target recognition."
\emph{Proc. of Computational Imaging V.}, 2007.

\item Dasgupta, S. ``Experiments with random projection." \emph{Arxiv Preprint arxiv}:1301.3849, 2013.

\item Dasgupta, S., Gupta, A. ``An elementary proof of the theorem of Johnson and Lindenstrauss." \emph{Random Structures and Algorithms},
22(1):60-65, 2003.

\item Davenport, M., Boufounos, P.T., Wakin, M., Baraniuk, R. ``Signal processing with compressive measurements." \emph{Selected Topics in Signal Processing, IEEE Journal of}, 4(2):445-460, 2010.

\item Dunson, D.B., Watson, M., Taylor, J.A. ``Bayesian latent variable models for median regression on multiple regression." \emph{Biometrics}, 59:296-304.

\item Figueiredo, M.A.T. ``Adaptive sparseness for supervised learning." \emph{IEEE Transactions on Pattern Analysis and Machine Intelligence}, 25(9):1050-1059, 2003.

\item Faes, C., Ormerod, J.T. and Wand, M.P. ``Variational Bayesian inference for parametric and nonparametric regression with missing data."
     \emph{Journal of the American Statistical Association}, 106(495):959-971, 2011.

\item Griffin, J.E., and Brown, P.J. ``Bayesian adaptive Lassos with non-convex penalization." \emph{Technical Report}, 2007.

\item Griffin, J.E., and Brown, P.J. ``Inference with Normal-Gamma prior distributions in regression problems." \emph{Bayesian Analysis}, 5(1):171-188, 2010.

\item Ghosal, S., Ghosh, J.K., and Van Der Vaart, A.W. ``Convergence rates of posterior distributions." \emph{Annals of Statistics}, 28(2):500-531, 2000.

\item Girolami, M., Rogers, S. ``Variational Bayesian multinomial probit regression with Gaussian process priors." \emph{Neural Computation}, 18(8):1790-1817, 2006.

\item Ghosal, S., and Van Der Vaart, A.W. ``Entropies and rates of convergence for Bayes and maximum likelihood estimation for mixture of normal densities." \emph{Annals of Statistics}, 29(5):1233-1263, 2001.

\item Ghosal, S., and Van Der Vaart, A.W. ``Convergence rates of posterior distributions for non iid observations." \emph{Annals of Statistics}, 35(1):192-223, 2007.

\item Hans, C. ``Bayesian Lasso regression." \emph{Biometrika}, 96(4):835-845, 2009.

\item Jiang, W. ``Bayesian variable selection for high dimensional
generalized linear models: convergence
rates of the fitted densities." \emph{Annals of Statistics}, 35(4):1487-1511, 2007.



\item Johnson, W.B., Lindenstrauss, J. ``Extensions of Lipschitz maps into a Hilbert space,"
\emph{Contemp. Math}, 26:189–206, 1984.

\item  Lee, H.K.H., Taddy, M., Gray, G.A. ``Selection of a representative sample." \emph{Journal of Classification}, 27:41-53, 2008.

\item  Madigan, D., Raghavan, N.,  Dumouchel, W. ``Likelihood-based data squashing: A modeling approach to instance construction."
\emph{Data Mining and Knowledge Discovery}, 6:173-190, 2002.

\item Owen, A. ``Data squashing by empirical likelihood." \emph{Data Mining and Knowledge Discovery}, 7:101-113, 2003.

\item Olive, B.P., Durand, R. ``Heterogeneity in radiation-induced DNA damage and repair in
tumour and normal cells measured using the 'comet' assay,''
\emph{Radiation Research}, 112:86-94.

\item Ormerod, J.T. and Wand, M.P. ``Gaussian variational approximate inference for generalized linear mixed models."
   \emph{Journal of Computational and Graphical Statistics}, 21(1):2-17, 2012.

\item Park, T., Casella, G. ``The Bayesian Lasso." \emph{Journal of the American Statistical Association}, 103(482):681-686, 2008.


\item Reich, B.J., Bondell, H., and Li, L. ``Sufficient dimension reduction via Bayesian mixture modeling." \emph{Biometrics}, 67(3):886-895, 2011.

\item Raftery, A.E., Madigan, D., and Hoeting, J.A. ``Bayesian model averaging for linear regression models." \emph{Journal of the American Statistical Association}, 92(437):179-191, 1997.

\item Rodriguez, A., Dunson, D.B., and Taylor, J. ``Bayesian hierarchically weighted finite mixture models for samples of distributions." \emph{Biostatistics}, 10(1):155-171, 2009.

\item Strawn, N., Armagan, A., Saab, R., Carin, L., and Dunson, D.B. ``Finite sample posterior concentration in high dimensional regression." \emph{Arxiv Preprint arxiv}:1207.4854, 2012.

\item Tibshirani, R. ``Regression selection and shrinkage via the Lasso." \emph{Journal of the Royal Statis. Soc. Ser. B}, 58(1):267-288, 1996.

\item Tipping, M.E. ``Sparse Bayesian learning and the relevance vector machine." \emph{Journal of Machine Learning Research}, 1:211-244, 2011.

\item Titsias, M.K., Lawrence, N.D. ``Bayesian Gaussian process latent variable model." \emph{13th International Conference on Artificial Intelligence and Statistics (AISTAT)}, 2010.

\item Tokdar, S.T., Zhu, Y.M., Ghosh, J.K. ``Bayesian density regression with logistic Gaussian process and subspace projection."
\emph{Bayesian Analysis}, 5(2):319-344, 2010.

\item Wang, H., Xia, Y. ``Sliced regression for dimension reduction." \emph{Journal of the American Statistical Association}, 103(482):811-821, 2008.

\item West, M. ``On scale mixtures of normal distributions." \emph{Biometrika}, 74(3):646-648, 1987.

\item Zanten, H.V., Knapik, B., Van der Vaart, A. ``Bayesian inverse problems with Gaussian priors." \emph{Annals of Statistics}, 39(5):2626-2657, 2011.

\item Zhou, S., Lafferty, J. and Wasserman, L. ``Compressed and privacy-sensitive sparse regression." \emph{IEEE Transactions on Information Theory}, 55(2):846-866, 2009.

\item Zhao, P., and Yu, B. ``On model selection consistency of Lasso." \emph{Journal of Machine Learning Research}, 7:2541-2567, 2006.

\end{description}
\section{Appendix}
Define,
 \begin{eqnarray*}
 d(f,f_0)^2 &=& \int\int (\sqrt{f}-\sqrt{f_0})^2 \nu_y(dy)\nu_x(dx)\\
 d_t(f,f_0) &=& \frac{1}{t}\left\{\int\int f_0\left(\frac{f_0}{f}\right)^t\nu_y(dy)\nu_x(dx)-1\right\}\\
 d_0(f,f_0) &=& \int\int f_0 log\left(\frac{f_0}{f}\right)\nu_y(dy)\nu_x(dx)
 \end{eqnarray*}
  Assume ${\cal P}_n$ is a sequence of sets of probability densities. Let $N(\epsilon_n,{\cal P}_n)$ is the minimum number of Hellinger balls of radius $\epsilon_n$ needed to cover ${\cal P}_n$.

  Define the following conditions:
 \begin{enumerate}
 \item $\log\:N(\epsilon_n,{\cal P}_n)\leq n\epsilon_n^2$ for all large n
 \item $\pi({\cal P}_n^{c})\leq e^{-2n\epsilon_n^2}$ for all large n
 \item $\pi[f: d_t(f,f_0)\leq \frac{\epsilon_n^2}{4}]\geq e^{-n\epsilon_n^2/4}$ for all large n
 \end{enumerate}
 It has been proved in Jiang (2005) that
 \begin{proposition}
 If $n\epsilon_n^2\rightarrow\infty$ then under 1, 2, 3 (for some $t>0$), we have
 \begin{equation*}
 E_{f_0}\pi\left[d(f,f_0)>4\epsilon_n\given (y_i,\bx_i)_{i=1}^{n}\right]\leq 4e^{-n\epsilon_n^2\:\mbox{min}\{1/2,t/4\}}
 \end{equation*}
 \end{proposition}
 Our proof will be complete if we can show conditions 1, 2, 3 hold in our case with some $t$. We prove them for $t=1$.

 First we will state and prove two propositions which will be used subsequently to prove the main results.

 \begin{proposition}\label{proppr}
 Assume $\bbeta$ is assigned $\mbox{N}(\bzero,\bSigma_{\bbeta})$ apriori. Then
 \begin{equation*}
 P(|(\bPhi\bx)'\bbeta-\bx'\bbeta_0|<\Delta)> P(X-Y\geq 2),
 \end{equation*}
  where $X\sim Pois(\frac{\Delta_1}{2})$, $Y\sim Pois(\frac{\lambda}{2})$ with $\Delta_1=\frac{\Delta^2}{(\bPhi\bx)'\bSigma_{\bbeta}(\bPhi\bx)}$, $\lambda=\frac{(\bx'\bbeta_0)^2}{(\bPhi\bx)'\bSigma_{\bbeta}(\bPhi\bx)}$
 \end{proposition}
 \begin{proof}
 Note that $(\bPhi\bx)'\bbeta\sim N(0,(\bPhi\bx)'\bSigma_{\bbeta}(\bPhi\bx))$. This implies that $\frac{|(\bPhi\bx)'\bbeta-\bx'\bbeta_0|^2}{(\bPhi\bx)'\bSigma_{\bbeta}(\bPhi\bx)}\sim \chi_1^2(\lambda)$.
 Invoking the popular representation of noncentral $\chi^2$ density with mixture of gamma densities, one gets
 \begin{align}\label{eq:priorconc}
& P(|(\bPhi\bx)'\bbeta-\bx'\bbeta_0|<\Delta)= P\left(\frac{|(\bPhi\bx)'\bbeta-\bx'\bbeta_0|^2}{(\bPhi\bx)'\bSigma_{\bbeta}(\bPhi\bx)}<\Delta_1\right) =P(\chi_1^2(\lambda)\leq\Delta_1)\nonumber\\
&=\sum_{i=0}^{\infty}\frac{e^{-\frac{\lambda}{2}}(\frac{\lambda}{2})^i}{i!}P(Z_{1+2i}<\Delta_1)
 \end{align}
 where $Z_{1+2i}\sim\chi_{1+2i}^2$. Note that $P(Z_{1+2i}<\Delta_1)>P(Z_{2+2i}<\Delta_1)=P(G<\Delta_1)$, where
 $G\sim Gamma(\frac{1}{2},1+i)$. Assume, $H_{i1},...,H_{i(1+i)}\stackrel{iid}{\sim}exp(\frac{1}{2})$ for $i\in\mathbb{N}$.

 Consider a Poisson process with interarrival times $H_{i1},...,H{i(i+1)}$ respectively. If $N(\Delta_1)$ be the number of events upto time $\Delta_1$ then
 $N(\Delta_1)\stackrel{d}{\equiv}X$ (using results from the theory of Poisson process) and
 \begin{align*}
 \left\{H_{i1}+\cdots+H_{i(i+1)}<\Delta_1\right\}=\left\{N(\Delta_1)>i+1\right\}
 \end{align*}
 Using the above results, (\ref{eq:priorconc}) can be written as
 \begin{align*}
 &\sum_{i=0}^{\infty}\frac{e^{-\frac{\lambda}{2}}(\frac{\lambda}{2})^i}{i!}P(Z_{1+2i}<\Delta_1)
 >\sum_{i=0}^{\infty}\frac{e^{-\frac{\lambda}{2}}(\frac{\lambda}{2})^i}{i!}P(Z_{2+2i}<\Delta_1)\\
 &=\sum_{i=0}^{\infty}\frac{e^{-\frac{\lambda}{2}}(\frac{\lambda}{2})^i}{i!}P\left(\sum\limits_{k=1}^{1+i}H_{ik}<\Delta_1\right)
 =\sum_{i=0}^{\infty}\frac{e^{-\frac{\lambda}{2}}(\frac{\lambda}{2})^i}{i!}P(N(\Delta_1)>1+i)\\
 &=\sum_{i=0}^{\infty}\frac{e^{-\frac{\lambda}{2}}(\frac{\lambda}{2})^i}{i!}P(X>1+i)
 =\sum_{i=0}^{\infty}\frac{e^{-\frac{\lambda}{2}}(\frac{\lambda}{2})^i}{i!}P(X\geq 2+i)\\
 &=E_{Y\sim Pois(\frac{\lambda}{2})}\left[P(X\geq 2+i)\right]=P(X\geq 2+Y)=P(X-Y\geq 2)
 \end{align*}
 \end{proof}
 We seek to obtain proposition of similar spirit when $\beta_j\sim DE(1)$ i.i.d prior. Using the fact that the DE(1) distribution is the scale mixture of normal
 distribution, we have
 \begin{align*}
 \beta_j\given\tau_j^2 &\stackrel{ind}{\sim} N(0,\tau_j^2),\:\:j=1,...,m_n\\
 \pi(\tau_j^2) &\propto \frac{1}{2}\exp\left\{-\frac{\tau_j^2}{2}\right\},\:\:j=1,...,m_n
 \end{align*}
 Stacking them up, we obtain $\bbeta=(\beta_1,...,\beta_{m_n})'\sim N(\bzero,\bD_{\btau})$, where $\bD_{\btau}=\mbox{diag}(\tau_1^2,...,\tau_{m_n}^2)$.
 The next proposition comes in the same spirit as Proposition \ref{proppr}. The proof of this proposition follows along the same line as of Proposition \ref{proppr} and is, therefore, omitted.
 \begin{proposition}\label{proppr2}
 $P(|(\bPhi\bx)'\bbeta-\bx'\bbeta_0|<\Delta\given \tau_1^2,...,\tau_{m_n}^2)> P(X_2-Y_2\geq 2\given \tau_1^2,...,\tau_{m_n}^2)$, where $X_2\sim Pois(\frac{\Delta_2}{2})$, $Y_2\sim Pois(\frac{\lambda_2}{2})$ with $\Delta_2=\frac{\Delta^2}{(\bPhi\bx)'\bD_{\btau}(\bPhi\bx)}$, $\lambda_2=\frac{(\bx'\bbeta_0)^2}{(\bPhi\bx)'\bD_{\btau}(\bPhi\bx)}$
 \end{proposition}
\textbf{Proof of the Theorem \ref{theorem1}:}

\begin{proof}
We will check the three conditions with $t=1$. Let $b_n=\sqrt{8\tilde{B}_n n\epsilon_n^2}$

\textbf{condition 1:}  Assume ${\cal P}_n$ be the set of all densities that can be represented by $\bbeta$ s.t. $|\beta_j|\leq b_n,\:\forall\:j$. Lets consider $l^{\infty}$ balls of the form $(a_j-\delta,a_j+\delta)$ for each coordinate with the center of each ball inside ${\cal P}_n$.
 It is not difficult to see that one needs at most $(\frac{2b_n}{2\delta}+1)^{m_n}$ such  balls to cover ${\cal P}_n$.

 Let $f_u$ be any density in ${\cal P}_n$. Therefore, $\exists\:\bgamma$ s.t. $u=(\bPhi\bx)'\bgamma$, $|\gamma_j|\leq b_n$ and
 \begin{equation*}
 f_u(y)=\exp\left\{y a(u)+b(u)+c(y)\right\}.
 \end{equation*}
 Let $\gamma_j\in (\alpha_j-\delta,\alpha_j+\delta)\:\forall\:j$, s.t. $|\gamma_j-\alpha_j|\leq\delta$ and $|\alpha_j|\leq b_n$. Let, $v=(\bPhi\bx)'\balpha$ and
 \begin{equation*}
 f_v(y)=\exp\left\{y a(v)+b(v)+c(y)\right\}.
 \end{equation*}
 We will use the well known fact that $d(f_u,f_v)\leq\left\{\frac{1}{2} d_0(f_u,f_v)\right\}^{1/2}$ and bound on $d_0(f_u,f_v)$. Note that,
 \begin{align}\label{eq:glmfit}
 d_0(f_u,f_v)&=\int\int f_v \:log\left(\frac{f_v}{f_u}\right)\nu_y(dy)\nu_x(dx)\nonumber\\
  &=\int\left\{\int\left[y(a(v)-a(u))+(b(v)-b(u))\right]f_v \nu_y(dy)\right\}\nu_x(dx)\nonumber\\
  &=\int\left\{(a(v)-a(u))\left(-\frac{b'(v)}{a'(v)}\right)+(b(v)-b(u))\right\}\nu_x(dx)\nonumber\\
  &=\int (v-u)\left\{(a'(u_v)\left(-\frac{b'(v)}{a'(v)}\right)+b'(u_v)\right\}\nu_x(dx)
 \end{align}
 where the last step follows from mean value theorem, $u_v$ is an intermediate point between $u$ and $v$. Applying Cauchy-Schwartz inequality we have,
 \begin{equation*}
 |v-u|=|(\bPhi\bx)'(\bgamma-\balpha)|\leq ||\bPhi\bx||\:||\bgamma-\balpha||\leq \sqrt{m_n}||\bx||\delta
 \leq \sqrt{p_n m_n}\delta=\theta_n\delta
 \end{equation*}
 With a similar argument one can show that $|v|\leq b_n\theta_n$ and $|u|\leq b_n\theta_n$. Therefore, $|u_v|\leq b_n\theta_n$.
 Using these results and (\ref{eq:glmfit})
 \begin{equation*}
 \frac{1}{2}d_0(f_u,f_v)\leq \sup\limits_{|h|\leq b_n\theta_n}|a'(h)|\sup\limits_{|h|\leq b_n\theta_n}\left|\frac{b'(h)}{a'(h)}\right|\theta_n \delta,\:
 \end{equation*}
 Using the well known fact that $d(f_u,f_v)\leq\left\{\frac{1}{2} d_0(f_u,f_v)\right\}^{1/2}$ we obtain
 \begin{equation*}
 d(f_u,f_v)\leq \sqrt{\sup\limits_{|h|\leq b_n\theta_n}|a'(h)|\sup\limits_{|h|\leq b_n\theta_n}\left|\frac{b'(h)}{a'(h)}\right|\theta_n \delta}.
 \end{equation*}
 Choosing $\delta=\frac{\epsilon_n^2}{\sup\limits_{|h|\leq b_n\theta_n}|a'(h)|\sup\limits_{|h|\leq b_n\theta_n}\left|\frac{b'(h)}{a'(h)}\right|\theta_n}$, one gets $d(f_u,f_v)\leq \epsilon_n$. Therefore, the entropy
 of ${\cal P}_n$ is bounded above by
 \begin{align*}
 \left(1+\frac{b_n\theta_n}{\epsilon_n^2}\sup\limits_{|h|\leq b_n\theta_n}|a'(h)|\sup\limits_{|h|\leq b_n\theta_n}\left|\frac{b'(h)}{a'(h)}\right|\right)^{m_n}=\left(1-\frac{1}{\epsilon_n^2}+\frac{D(b_n\theta_n)}{\epsilon_n^2}\right)^{m_n}
 \leq\:\left(\frac{D(b_n\theta_n)}{\epsilon_n^2}\right)^{m_n}
 \end{align*}
 Therefore, $log\:N(\epsilon_n,{\cal P}_n)\leq m_n\:log\left(\frac{D(b_n\theta_n)}{\epsilon_n^2}\right)$.  Using the assumptions given in the proposition, condition 1 follows.

 \textbf{condition 2:}
 Note that
 \begin{equation*}
 \pi({\cal P}_n^{c})=\pi(\cup_{j=1}^{m_n}[|\beta_j|>b_n])\leq \sum\limits_{j=1}^{m_n}\pi(|\beta_j|>b_n).
 \end{equation*}
 By Mills ratio this quantity is bounded above by $2m_n\:\frac{\exp\left\{-b_n^2/2\tilde{B}_n^2\right\}}{\sqrt{2\pi b_n^2/\tilde{B}_n}}=2m_n\frac{\exp\{-4n\epsilon_n^2\}}{\sqrt{2\pi n\epsilon_n^2}}$. Using conditions in (i), for all large n,
 the expression can further be bounded above by $\exp\left\{-2n\epsilon_n^2\right\}$. This yields condition 2.

 \textbf{condition 3:}
  Using proposition \ref{proppr} we obtain, $P(|(\bPhi\bx)'\bbeta-\bx'\bbeta_0|<\Delta)> P(X-Y\geq 2)$. Note that, when $X\sim Pois(\frac{\Delta_1}{2})$ and
 $Y\sim Pois(\frac{\lambda_1}{2})$, $X-Y$ follows Skellam distribution with
 \begin{equation}\label{eq:skel}
 P(X-Y=k)=\exp\{-(\lambda_1+\Delta_1)\}\left(\frac{\Delta_1}{\lambda_1}\right)I_{|k|}\left(2\sqrt{\lambda_1\Delta_1}\right).
 \end{equation}
 Plugging in $\lambda_1$ and $\Delta_1$ in (\ref{eq:skel}), we have
 \begin{align*}
 P(X-Y=k)=\exp\left\{-\frac{((\bx'\bbeta_0)^2+\Delta^2)}{(\bPhi\bx)'\bSigma_{\bbeta}(\bPhi\bx)}\right\}\left(\frac{\Delta^2}{(\bx'\bbeta_0)^2}\right)^{k/2}I_{|k|}\left(2\frac{\Delta
 |\bx'\bbeta_0|}{(\bPhi\bx)'\bSigma_{\bbeta}(\bPhi\bx)}\right)
 \end{align*}
 Now we use the fact that for $z>0$, $I_{\nu}(z)>2^{\nu}z^{\nu}\bGamma(\nu+1)$ (Joshi et al, 1991) to get
 \begin{align*}
 P(X-Y\geq 2) &>P(X-Y=2)>\exp\left\{-\frac{((\bx'\bbeta_0)^2+\Delta^2)}{(\bPhi\bx)'\bSigma_{\bbeta}(\bPhi\bx)}\right\}\frac{2^4 \Delta^4}{[(\bPhi\bx)'\bSigma_{\bbeta}(\bPhi\bx)]^2}\\
 &>\exp\left\{-\frac{((\bx'\bbeta_0)^2+\Delta^2)}{\underline{B}_n||\bPhi\bx||^2}\right\}\frac{2^4 \Delta^4}{\tilde{B}_n^2||\bPhi\bx||^4}\\
 &>\exp\left\{-\frac{((\bx'\bbeta_0)^2+\Delta^2)\log(m_n)}{B_1||\bPhi\bx||^2}\right\}\frac{2^4 \Delta^4}{m_n^{2v}B^2||\bPhi\bx||^4}
 \end{align*}
 Take $\Delta=\frac{\epsilon_n^2}{4\eta}$, $\eta$ will be chosen later.
 Note that,
 $2v\log{m_n}+2\log(B)-4\log(2)-4\log(\Delta)>0,\:\:$ for all large $n$ and
 \begin{equation*}
 8\frac{2v\log{m_n}+2\log(B)-4\log(2)-4\log(\Delta)}{n\epsilon_n^2}\rightarrow 0,
 \end{equation*}
  from the assumptions. Therefore, $\frac{2^4 \Delta^4}{m_n^{2v}B^2||\bPhi\bx||^4}>\exp\left\{-n\epsilon_n^2/8\right\}$ for all $\bx=\bx_1,...,\bx_n$.

 Also, note that $(\bx'\bbeta_0)^2<\sum_n |\beta_{j0}|<K$. Therefore,
 \begin{align*}
& \exp\left\{-\frac{((\bx'\bbeta_0)^2+\Delta^2)\log(m_n)}{B_1||\bPhi\bx||^2}\right\}>\exp\left\{-\frac{(K^2+\Delta^2)\log(m_n)}{B_1||\bPhi\bx||^2}\right\}\\
&>\exp\left\{-\frac{n\epsilon_n^2}{8}\frac{(K^2+1)\log(m_n)}{B_1||\bPhi\bx||^2}\right\}>\exp\left\{-\frac{n\epsilon_n^2}{8}\right\}
 \end{align*}
 Therefore, $P(X-Y\geq 2)>\exp\{-\frac{n\epsilon_n^2}{4}\}$. Proposition \ref{proppr} then yields
 \begin{equation*}
 P\left(|(\bPhi\bx)'\bbeta-\bx'\bbeta_0|<\frac{\epsilon_n^2}{4\eta}\right)>\exp\left\{-\frac{n\epsilon_n^2}{4}\right\}\:\:\mbox{for all large}\:n
 \end{equation*}
 For $\bx=\bx_1,...,\bx_n$,
 let $\mathcal{S}=\left\{\bbeta:|(\bPhi\bx)'\bbeta-\bx'\bbeta_0|<\frac{\epsilon_n^2}{4\eta}\right\}$. Now, take $t=1$. Therefore we can write $d_t(f,f_0)=E_x[g(u_i)((\bPhi\bx)'\bbeta-\bx'\bbeta_0)]$, where $g$ has continuous derivative in the neighborhood of $\bx'\bbeta_0$ and $u_i$ is an intermediate point between $(\bPhi\bx)'\bbeta$ and $\bx'\bbeta_0$, by integrating $y$ and applying a first order Taylor expansion. Choose $\eta$ s.t. $|g|<\eta$ in the neighborhood $[-(K+1),(K+1)]$ for all large $n$. Then
 \begin{equation*}
 |u_i|<|(\bPhi\bx)'\bbeta-\bx'\bbeta_0|+|\bx'\bbeta_0|<\frac{\epsilon_n^2}{4\eta}+K
 \end{equation*}
 implies that $d_t(f,f_0)<\frac{\epsilon_n^2}{4}$ is a subset of $\mathcal{S}$ . Hence, condition 3 is satisfied. 
 \end{proof}

 \textbf{Proof of the Theorem \ref{theorem2}:}

 \begin{proof}

 Here also we proceed by checking conditions 1, 2, 3 for $t=1$.

 \textbf{condition 1:} This proof follows exactly along the same line as in the proof of the condition 1. in Theorem \ref{theorem1}.

 \textbf{condition 2:} Note that,
 \begin{equation*}
 \pi({\cal P}_n^{c})=\pi(\cup_{j=1}^{m_n}[|\beta_j|>b_n])\leq \sum\limits_{j=1}^{m_n}\pi(|\beta_j|>b_n)=m_n\exp\{-b_n\}.
 \end{equation*}
Under the assumptions, this expression can be bounded by $\exp\left\{-2n\epsilon_n^2\right\}$ for all large n.

\textbf{condition 3:} Following the same line as in the last Proposition, we obtain
 \begin{align*}
 P(|(\bPhi\bx)'\bbeta-\bx'\bbeta_0|<\Delta\given \tau_1^2,...,\tau_{m_n}^2) &> P(X_2-Y_2\geq 2\given \tau_1^2,...,\tau_{m_n}^2)\\
 &>\exp\left\{-\frac{(\bx'\bbeta_0)^2+\Delta^2}{\tau_{min}^2||\bPhi\bx||^2}\right\}\frac{2^4\Delta^4}{||\bPhi\bx||^4\tau_{max}^4}
 \end{align*}
 Integrating w.r.t $(\tau_{min}^2,\tau_{max}^2)$
 \begin{align*}
  P(|(\bPhi\bx)'\bbeta-\bx'\bbeta_0|<\Delta)&=\int_{0}^{\infty}\int_{\tau_{min}^2}^{\infty}
  \exp\left\{-\frac{(\bx'\bbeta_0)^2+\Delta^2}{\tau_{min}^2||\bPhi\bx||^2}\right\}\frac{2^4\Delta^4}{||\bPhi\bx||^4\tau_{max}^4}f(\tau_{min}^2,\tau_{max}^2)
  d\tau_{max}^2 d\tau_{min}^2\\
 &>\int_{1}^{\infty}\int_{\tau_{min}^2}^{\infty}
  \exp\left\{-\frac{(\bx'\bbeta_0)^2+\Delta^2}{\tau_{min}^2||\bPhi\bx||^2}\right\}\frac{2^4\Delta^4}{||\bPhi\bx||^4\tau_{max}^4}f(\tau_{min}^2,\tau_{max}^2)
  d\tau_{max}^2 d\tau_{min}^2\\
 &> \exp\left\{-\frac{(\bx'\bbeta_0)^2+\Delta^2}{||\bPhi\bx||^2}\right\}\frac{2^4\Delta^4}{||\bPhi\bx||^4}
 \int_{1}^{\infty}\int_{\tau_{min}^2}^{\infty}\frac{1}{\tau_{max}^4}f(\tau_{min}^2,\tau_{max}^2)
  d\tau_{max}^2 d\tau_{min}^2\\
 \end{align*}
 Note that
 \begin{align*}
 &\int_{1}^{\infty}\int_{\tau_{min}^2}^{\infty}\frac{1}{\tau_{max}^4}f(\tau_{min}^2,\tau_{max}^2)
  d\tau_{max}^2 d\tau_{min}^2\\
  &=\int_{1}^{\infty}\int_{\tau_{min}^2}^{\infty}{m_n\choose 2}\frac{1}{2^2\tau_{max}^4}\exp\left\{-\frac{(\tau_{max}^2+\tau_{min}^2)}{2}\right\}
  \left[\exp\left\{-\frac{\tau_{min}^2}{2}\right\}-\exp\left\{-\frac{\tau_{min}^2}{2}\right\}\right]^{m_n-2}d\tau_{max}^2 d\tau_{min}^2\\
  &=\int_{1}^{\infty}\frac{m_n}{2}\frac{1}{2^2\tau_{max}^4}\exp\left\{-\frac{\tau_{max}^2}{2}\right\}\left[\exp\left\{-\frac{1}{2}\right\}-\exp\left\{-\frac{\tau_{min}^2}{2}\right\}\right]
  d\tau_{max}^2 d\tau_{min}^2\\
  &<\int_{1}^{\infty}\frac{m_n}{2}\exp\left\{-\frac{\tau_{max}^2}{2}\right\}\left[\exp\left\{-\frac{1}{2}\right\}-\exp\left\{-\frac{\tau_{min}^2}{2}\right\}\right]
  d\tau_{max}^2 d\tau_{min}^2\\
  &=\exp\left\{-\frac{m_n}{2}\right\}.
 \end{align*}
Let, $C_n=\int_{1}^{\infty}\int_{\tau_{min}^2}^{\infty}\frac{1}{\tau_{max}^4}f(\tau_{min}^2,\tau_{max}^2)
  d\tau_{max}^2 d\tau_{min}^2$. Therefore
\begin{align*}
P(|(\bPhi\bx)'\bbeta-\bx'\bbeta_0|<\Delta)>\exp\left\{-\frac{(\bx'\bbeta_0)^2+\Delta^2}{||\bPhi\bx||^2}\right\}\frac{2^4\Delta^4}{||\bPhi\bx||^4}C_n
\end{align*}
Take $\Delta=\frac{\epsilon_n^2}{4\eta}$.
 Note that
  for all large $n$, $\bx=\bx_1,...,\bx_n$.
 \begin{equation*}
 8\frac{4\log{||\bPhi\bx||}-4\log(2)-4\log(\Delta)-log(C_n)}{n\epsilon_n^2}\rightarrow 0,
 \end{equation*}
  from the assumptions. Therefore, $\frac{2^3 \Delta^4 C_n}{||\bPhi\bx||^4}>\exp\left\{-n\epsilon_n^2/8\right\}$.

 Also, note that $(\bx'\bbeta_0)^2<\sum_n |\beta_{j0}|<K$. Therefore,
 \begin{align*}
& \exp\left\{-\frac{((\bx'\bbeta_0)^2+\Delta^2)}{||\bPhi\bx||^2}\right\}>\exp\left\{-\frac{(K^2+\Delta^2)}{||\bPhi\bx||^2}\right\}\\
&>\exp\left\{-\frac{n\epsilon_n^2}{8}\right\}
 \end{align*}
 Proposition \ref{proppr2} then yields
 \begin{equation*}
 P\left(|(\bPhi\bx)'\bbeta-\bx'\bbeta_0|<\frac{\epsilon_n^2}{4\eta}\right)>\exp\left\{-\frac{n\epsilon_n^2}{4}\right\}\:\:\mbox{for all large}\:n
 \end{equation*}
 Let $\mathcal{S}=\left\{\bbeta:|(\bPhi\bx)'\bbeta-\bx'\bbeta_0|<\frac{\epsilon_n^2}{4\eta}\right\}$. Now, take $t=1$, therefore we can write $d_t(f,f_0)=E_x[g(u_i)((\bPhi\bx)'\bbeta-\bx'\bbeta_0)]$, where $u_i$ is an intermediate point between $(\bPhi\bx)'\bbeta$ and $\bx'\bbeta_0$, by integrating $y$ and applying a first order Taylor expansion. Now choose $\eta$ as in the proof of the previous theorem. Then
 \begin{equation*}
 |u_i|<|(\bPhi\bx)'\bbeta-\bx'\bbeta_0|+|\bx'\bbeta_0|<\frac{\epsilon_n^2}{4\eta}+K
 \end{equation*}
implies that $d_t(f,f_0)<\frac{\epsilon_n^2}{4}$
is a subset of $\mathcal{S}$. Hence, condition 3 is satisfied. 
 \end{proof}
\end{document}